\documentclass{article}
\usepackage{amssymb,bm}
\usepackage{amsmath}
\usepackage{amsthm}
\usepackage{amsfonts}
\usepackage{mathrsfs}
\usepackage{appendix}
\usepackage{hyperref}
\usepackage{authblk}
\usepackage{color}

\usepackage{mathtools}

\newtheorem{thm}{Theorem}[section]

\newtheorem{theorem}{Theorem}[section]
\newtheorem{lemma}[theorem]{Lemma}

\theoremstyle{remark}

\usepackage{graphics}
\usepackage{graphicx,subcaption}
\usepackage{algorithm}
\usepackage{algorithmic}
\hypersetup{colorlinks,linkcolor={blue},citecolor={blue},urlcolor={red}}

\usepackage{cite}
\usepackage[top=2.5cm, bottom=2.5cm, left=2.5cm, right=2.5cm]{geometry}

\usepackage{threeparttable}
\usepackage{float}
\floatstyle{plaintop}
\restylefloat{table}
\def\var{\rm{var}}

\floatstyle{plain}
\restylefloat{figure}
\usepackage{placeins}

\newcommand{\mcP}{\mathcal{P}}
\newcommand{\mcL}{\mathcal{L}}
\newcommand{\mcN}{\mathcal{N}}

\newcommand{\mbR}{\mathbb{R}}
\newcommand{\mbE}{\mathbb{E}}
\newcommand{\mbP}{\mathbb{P}}

\newcommand{\tr}{\rm{Tr}}

\date{}
\begin{document}

\title{Derivation of Information-Theoretically Optimal Adversarial Attacks with Applications to Robust Machine Learning}

\author{
Jirong Yi,  
Raghu Mudumbai, Weiyu Xu {\footnote{
Department of Electrical and Computer Engineering, University of Iowa, Iowa City, IA 52242. Corresponding email: weiyu-xu@uiowa.edu. This work was partially presented in the ECE department seminar of Iowa State University (November 22nd, 2019), and  Transdisciplinary Research Workshop in Principles of Data Science at University of Iowa (July 17th, 2020).}}\,}\,

\maketitle
\begin{abstract}

We consider the theoretical problem of designing an optimal adversarial attack on a decision system that maximally degrades the achievable performance of the system as measured by the mutual information between the degraded signal and the label of interest. This problem is motivated by the existence of adversarial examples for machine learning classifiers. By adopting an information theoretic perspective, we seek to identify conditions under which adversarial vulnerability is unavoidable i.e. even optimally designed classifiers will be vulnerable to small adversarial perturbations. We present derivations of the optimal adversarial attacks for discrete and continuous signals of interest, i.e., finding the optimal perturbation distributions to minimize the mutual information between the degraded signal and a signal following a continuous or discrete distribution. In addition, we show that it is much harder to achieve adversarial attacks for minimizing mutual information when multiple redundant copies of the input signal are available. This provides additional support to the recently proposed ``feature compression" hypothesis as an explanation for the adversarial vulnerability of deep learning classifiers. We also report on results from computational experiments to illustrate our theoretical results.

\end{abstract}


\section{Introduction}\label{Sec:Introduction}

Deep learning methods have revolutionized many data processing applications that had previously been considered intractable such as computer vision, natural language processing and speech recognition \cite{goodfellow_deep_2016,nitin_bhagoji_practical_2018,du_towards_2018,cheng_query-efficient_2018,bubeck_adversarial_2018-1,bubeck_adversarial_2018}. However, deep learning systems have been shown to be vulnerable to {\it adversarial attacks} \cite{szegedy_intriguing_2013,goodfellow_explaining_2014,carlini_towards_2017,kos_adversarial_2018,xie_information-theoretic_2019,chelombiev_adaptive_2019,uesato_are_2019,yi_trust_2019,wiyatno_adversarial_2019,schmidt_adversarially_2018,szegedy_intriguing_2013,papernot_distillation_2016,kurakin_adversarial_2016,shafahi_universal_2018}. Specifically, it has been shown that the outputs of many deep learning systems can be manipulated with imperceptibly small perturbations applied to the inputs \cite{lai_adversarial_2020,chen_boundary_2019,andriushchenko_square_2019,kos_adversarial_2018,schmidt_adversarially_2018,madry_towards_2017}. We will use the term ``adversarial fragility'' to describe this vulnerability to adversarial attacks.

A common theme in the previous literature is to explain adversarial fragility as a consequence of some deficiency in practical machine learning systems with the implication that this fragility can be eliminated by modifying the design of the system to remove the deficiency \cite{kos_adversarial_2018,schmidt_adversarially_2018,madry_towards_2017,fawzi_adversarial_2018}. In this paper, we take a different view: we consider estimation or classification problems where even the {\it theoretically optimal decision system} is vulnerable to adversarial attack. Thus we treat adversarial fragility as a function of the statistics of the decision problem rather than as an artifact of suboptimal design.

We define adversarial attacks as perturbations that maximally degrade the information contained in an input signal as measured by the mutual information between the signal and the label (or quantity) of interest, a first information-theoretically optimal attack. Note that this definition {\it makes no reference to a particular machine learning system}. One consequence of this definition is that adversarial attacks are inherently {\it transferable}: the reduction of mutual information caused by the adversarial attack degrades the performance of any conceivable decision system that relies on that input signal to the system. While adversarial attacks have been empirically shown to transfer among different machine learning systems, such transferability is entirely unintentional and the phenomenon is not well-understood.

With our definition, we can generalize the tools of rate-distortion theory \cite{cover_elements_2012} to  study adversarial attacks. Rate distortion theory is a branch of information theory that studies the amount of information loss that can be caused by a certain amount of distortion to signal of interest,  with applications to data compression. The efficiency of the data compression is maximized when the signal distortion is chosen in such a way to maximize the resulting information loss, thus effectively reducing the amount of information that must be preserved. It turns out that maximizing information loss is also a good way to model what an adversarial attacker on a decision system would seek to do. Thus the rate distortion theory provides a natural mathematical framework to study adversarial attacks. However, different from problem settings in traditional rate distortion theory, the adversarial attackers are not directly modifying the label (or quantity) of interest, but instead are modifying data generated from or related to the label (or quantity) of interest. 

We present a formal definition of adversarial attacks on classification or estimation algorithms as a generalization of the classic rate distortion problem. We then present derivations of the optimal adversarial attacks on two classic decision problems with constraints on the attack size: (a) estimating a vector of continuous random variables each of which represents an independent observation of a common source variable with an Gaussian prior distribution, and (b) estimating discrete  random variables. We also consider a variant of this problem where the adversarial attacks are limited to attacking only a subset of the variables being estimated.

One important finding from our results is that adversarial attacks are significantly less effective when a large number of independent observations of the source variable are available. In other words, when there is redundant information about the source data available to the decision algorithm, it can be much more robust to attacks. Conversely, adversarial attacks can have significantly more dramatic effects when redundant observations of the source variable are not available. This lends support to our recently proposed ``feature compression'' hypothesis \cite{yi_trust_2019,xie_information-theoretic_2019} as an explanation for the adversarial fragility of deep learning systems. Under this hypothesis, deep learning systems are vulnerable to adversarial attacks because they compress their data into a minimal number of features that contain enough information about the source data to allow for sufficiently accurate classification under no adversarial attacks. In other words, under the ``feature compression'' hypothesis, deep learning systems are systematically blind to all non-essential features even if they may contain information relevant to the source label because this information is redundant. Such compression is useful for generalization, but at the cost of robustness: it makes the system vulnerable to attacks that are narrowly targeted at the minimal features it relies on. This intuition is confirmed by our results in this paper.

Our main contributions are summarized as follows.

\begin{itemize}

\item We present a general and formal definition of adversarial attacks as an constrained optimization problem for maximizing the information loss for a given attack size, giving an information-theoretically optimal attack. 

\item We show that the optimal adversarial attack problem is a generalization of the classical rate-distortion problem, and present characterizations of the optimal solution for both binary and Gaussian decision problems.

\item We show that adversarial attacks are significantly harder against decision systems whose inputs contain many redundant copies of information.

\end{itemize}

This paper is organized as follows. In Section \ref{sec:problemformulation}, we formally formulate our problem model in a general setting which is further specified under different scenarios in Section \ref{Sec:L2MetricAttack} and \ref{Sec:L1MetricAttack}. We give the main theoretical results in these two sections, and present experimental results for supporting our theoretical results in Section \ref{Sec:ExperimentalResults}.

{\bf Notations:} We denote the distribution of a random variable by $p(X)$, and the probability of a realization of random variable $X$, i.e., $X=x$, by $P(X=x)$. The sample space of a random variable $X$ is denoted by $\Omega_X$. We use $G_{ij}$ to denote the element of a matrix $G$ in the $i$-th row and $j$-th column. The $\tr(F)$ refers to the trace of $F$.

\section{Problem Formulation}
\label{sec:problemformulation}

We use random variable $U$  to denote the quantity (or label) of interest, and use random variable $X\in\mbR$ to denote the data generated  (or called ``observation synthesized'') from $U$. We are interested in adding random perturbation $E$ to $X$, producing random variable $Y=X+E$, such that the mutual information between $U$ and $Y=X+E$ is minimized. Since adding perturbations often introduces costs, we put constraints on the perturbation $E$ . For example, to make the perturbation less perceptible, one can require the perturbation to be smaller than $\epsilon$ in the expectation of its $\ell_{2}$ norm. The overall framework of our proposed model is illustrated in Figure \ref{Fig:RateDistorAdvAttack}.

\begin{figure}[!htb]
\centering
\includegraphics[width=\textwidth]{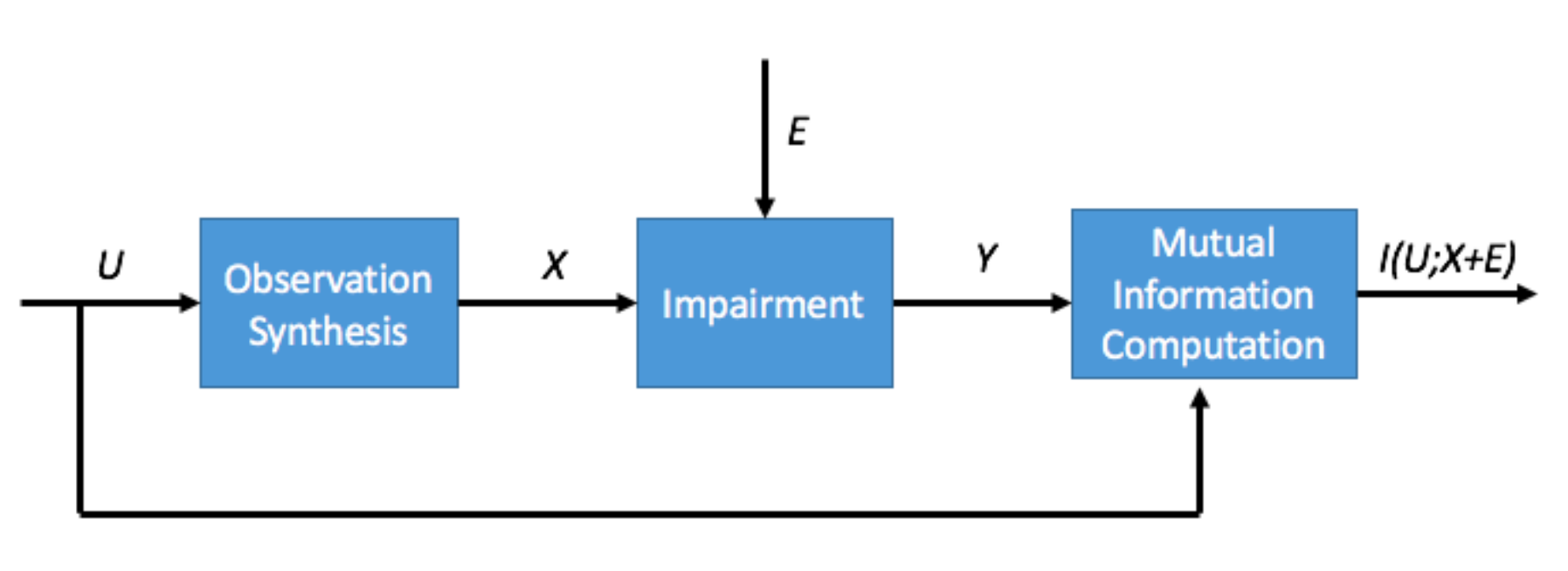}
\caption{Mutual information-based framework for adversarial attacks. A data-generating (observation synthesis) process gives $X=U+W$ where $W$ characterizes the noise.}\label{Fig:RateDistorAdvAttack}
\end{figure}

Let $(U,X)$ and $(U,X,E)$ be jointly distributed according to distributions $p(U,X)$ and $p(U,X,E)$, respectively. We can formulate the problem as
\begin{align}\label{Defn:MutualInfoMinimization}
\min_{p(E|U,X) \in \mathcal{P}} I(U; X+E)
\end{align}
where $\mathcal{P}$ is a set of admissible probability distributions. For example, if we want to restrict the perturbation in terms of average $\ell_{2}$ norm being smaller than a certain threshold $\epsilon$, we have 

\begin{align}\label{Defn:AdmissibleProbDist}
\mcP:=\{p(E| U, X): \mbE[\|E\|^2]\leq \epsilon, X\in\Omega\}.
\end{align}

\begin{lemma}\label{Lem:OptimizationConvexity}
If $\mcP$ is a convex set, the optimization problem (\ref{Defn:MutualInfoMinimization}) is a convex optimization problem.
\end{lemma}

\begin{proof}
	(of Lemma \ref{Lem:OptimizationConvexity}) We will show that the $I(U;X+E)$ is convex with respect to $p(E|U,X)$.  From \cite{cover_elements_2012}, we know $I(U;X+E)$ is convex with respect to $p(X+E|U)$.  We can see that, if $\mcP$ is convex, the set of conditional distributions for $p(X+E|U)$ is also convex.  This proves that the optimization problem (\ref{Defn:MutualInfoMinimization}) is a convex optimization problem.
	  
\end{proof}

The convexity of the optimization problem (\ref{Defn:MutualInfoMinimization}) opens the door to efficiently calculating the information-theoretically attack

\section{Adversarial Attack for Full Set of Random Variables}\label{Sec:L2MetricAttack}

In this section, we will consider the problem \eqref{Defn:MutualInfoMinimization} in scenarios where $U$ follows a continuous distribution, or a discrete distribution.

\subsection{The Case of Gaussian Random Variables}\label{Sec:GeneralGaussian}

We first consider the problem \eqref{Defn:MutualInfoMinimization} for random variables following continuous distributions, i.e., finding optimal adversarial perturbation random variable $E$ to minimize the mutual information between $U$ and $X+E$.  The main result is presented in Theorem \ref{Thm:SingleGaussianVar}.

\begin{thm}\label{Thm:SingleGaussianVar}
Suppose that random variable $U\in\mbR$ follows Gaussian distribution $\mathcal{N}(0, a^{2})$. Let random variable $W$ follow Gaussian distribution $\mathcal{N}(0, \sigma^{2})$, and be independent of $U$.  The data $X=U+W$. Let $D$ be a non-negative number. 
Under the constraint that the perturbation $E$ satisfies $\mathbb{E} \{E^2\} \leq D$, the optimal perturbation $E$ that minimizes the mutual information $I(U;U+W+E)$, is jointly Gaussian distributed with $U$ and $W$, and has zero mean. Moreover, under the optimal distribution, the covariance matrix of $(U, W, E)$ is given by
\begin{align}
F=\begin{bmatrix}
  a^{2}& 0 & x\\
0& \sigma^{2}  &y\\
x& y& D\\
 \end{bmatrix},
 \end{align}
where ($x$, $y$) are the optimal solution to the following optimization problem (\ref{Defn:OptimizationforGaussian}), i.e.,
\begin{align}\label{Defn:OptimizationforGaussian}
&\min_{x,~y} ~~~~~~\frac{(a^2+x)^{2}}{a^2(a^{2}+\sigma^{2}+D+2y+2x)},\nonumber\\
&{\rm s.t.}~~~~~~  \frac{x^{2}}{a^{2}}+\frac{y^{2}}{\sigma^{2}} \leq D.
\end{align}
In addition, the smallest achievable mutual information is given by $-\frac{1}{2} \log\left(1-\frac{(a^2+x)^{2}}{a^2(a^{2}+\sigma^{2}+D+2y+2x)}\right)$.
\end{thm}

\begin{proof}
(of Theorem \ref{Thm:SingleGaussianVar}) Let us assume that $F$ is the covariance matrix of $(U,W, E)$, then the covariance matrix of $(U+W+E, U)$ is given by
\begin{align}
G=\begin{bmatrix}
  a^{2}+\sigma^{2}+D+2y+2x & a^{2}+x\\
a^{2}+x& a^{2}
 \end{bmatrix}.
 \end{align}

 Define $Z=U-\frac{G_{21}}{G_{11}} (U+W+E)$. Then the variance of $Z$ is given by ${\var}(Z)= G_{22}- \frac{G_{21}}{G_{11}} G_{12}$. Then we have the following argument:
 \begin{align*}
I(U+W+E;U)
& =h(U)-h(U|U+W+E)\\
 & = h(U)-h(Z|U+W+E)\\
 & \geq h(U)-h(Z) \\
 &\geq h(U)-h(Z_{g})\\
 &=\frac{1}{2} \log(2\pi ea^{2})-\frac{1}{2} \log(  2\pi e (G_{22}- G_{21} G_{11}^{-1} G_{12}  )        )\\
 &=\frac{1}{2} \log( 2\pi e a^{2})
 - \frac{1}{2} \log\left(  2\pi e \left(a^{2}-\frac{(a^{2}+x)^{2}   } {a^{2}+\sigma^{2}+D+2y+2x}  \right)\right)\\
 &=-\frac{1}{2} \log\left(1-   \frac{(a^2+x)^{2}}{a^2(a^{2}+\sigma^{2}+D+2y+2x)} \right),
 \end{align*}
 where $Z_{g}$ is a Gaussian random variable with the same mean and variance as those of $Z$, and we used the fact that the maximum value of $h(Z)$ is achieved when $Z$ follows the Gaussian distribution.

Moreover, when $(U,V,W)$ are zero-mean jointly Gaussian with covariance matrix $F$, the mutual information between $I(U;U+W+E)$ can achieve $-\frac{1}{2} \log\left(1-   \frac{(a^2+x)^{2}}{a^2(a^{2}+\sigma^{2}+D+2y+2x)}\right)$.
The smallest achievable value of $I(U;U+W+E)$ is thus given by the following optimization problem:
\begin{align}\label{Eq:IntermediateOptimization}
&\min_{x,~y} ~~~~~~-\frac{1}{2} \log\left( 1-   \frac{(a^2+x)^{2}}{a^2(a^{2}+\sigma^{2}+D+2y+2x)}\right),\\
&{\rm s.t.}~~~~~~\begin{bmatrix}
  a^{2}& 0 & x\\
0& \sigma^{2}  &y\\
x& y& D
 \end{bmatrix}\succcurlyeq 0 .
\end{align}

By the Schur complement, the condition that $F\succcurlyeq 0$ is equivalent to $\frac{x^{2}}{a^{2}}+\frac{y^{2}}{\sigma^{2}} \leq D$. From the monotonicity of $-\frac{1}{2}\log(1-x)$ in $x$, the solution minimizing the objective function in \eqref{Eq:IntermediateOptimization} will also be the minimizer of the objective function \eqref{Defn:OptimizationforGaussian}. 
\end{proof}

Theorem \eqref{Thm:SingleGaussianVar} characterizes the optimal adversarial perturbation $E$ for minimizing the mutual information between $U+W+E$ and the Gaussian random variable $U$, i.e., the optimal $E$ must follow a joint Gaussian distribution with source $U$ and random noise $W$. In addition, it provides a feasible way to compute the distribution of $E$, and to compute the minimum achievable mutual information achieved, i.e., by solving the optimization problem \eqref{Defn:OptimizationforGaussian}.

\subsection{Linear Projections of Gaussian Random Variables}\label{Sec:GaussianWithProjection}

In this section, we consider the case where the quantity of interest is a Gaussian random vector $U\in\mbR^m$, and the data generating process is modeled by linear projections.  We show that as the dimensionality of the projected space increases, the minimized mutual information under a given attack budget can increase. Since the linear projection can be interpreted as creating multiple copies of the original source $U$, this implies that adding redundant copies better preserves the mutual information under adversarial attacks, and makes it harder for the attacker to perform adversarial attacks. The result is presented in Theorem \ref{Thm:GaussianMultipleCopies}.

\begin{thm}\label{Thm:GaussianMultipleCopies}
Let  $U_{1}$, $U_{2}$, $\cdots$, and $U_{m}\in\mbR$ be $m$ independent Gaussian random variables each following Gaussian distribution $\mathcal{N} (0,1)$.  Let $X=H U$, where $H \in \mathbb{R}^{n \times m}$  is a given matrix, and $U=[U_1\ U_2\ \cdots\ U_m]^T$. Let $Y=H U +E$ be the input data to decision systems, where $E$ is the perturbation.   Let  $k$ be the rank of $H$, and let $\sigma_{1}$, $\sigma_{2}$,..., and $\sigma_{k}$ be the singular values of matrix $H$. Suppose $H$ has a singular value decomposition  $H=Q C V^{T}$, where $Q \in \mathbb{R}^{n \times k}$,  $C \in \mathbb{R}^{k \times k}$,  $V \in \mathbb{R}^{m \times k}$, and $C$ is a diagonal matrix whose diagonal elements are $\sigma_{1}$, $\sigma_{2}$,..., and $\sigma_{k}$.

 Consider a perturbation budget $D\leq \sum_{i=1}^{k} \sigma_{i}$ such that $\mbE[\|E\|^{2}]\leq D $. Then there exists a positive number $\tau$ such that
 $$
 D_{i}=
 \begin{cases}
 \sigma_{i}^2, \sigma_{i}^2 \leq \tau, \\
 \tau, \sigma_{i}^{2} \geq \tau
 \end{cases},
 \quad \sum_{i=1}^{k}  D_{i}=D.
 $$
The smallest objective value (mutual information) of optimization problem(\ref{Defn:MutualInfoMinimization}) is given by $$I(U;Y) = \frac{1}{2}\sum_{i=1}^{k} \log\left(\frac{ \sigma_{i}^{2}}{D_{i}} \right)$$. An optimal perturbation $E^{*}$ which achieves this smallest mutual information $I(U;Y)$ is given by $E^*=Q \Lambda$, where $\Lambda \in \mathbb{R}^{k}$ is a vector of $k$ independent Gaussian random variables following the distribution $\mcN(\bm{0},G_\Lambda)$ with $G_\Lambda$ being a diagonal matrix with $D_1,\cdots,D_k$ on the diagonal. Moreover, $\Lambda_{i}$ is chosen such that $(\sigma_{i}V_{i}^{T} U- \Lambda_{i})$ is independent of $\Lambda_{i}$, where $V_{i}$ is the $i$-th column of $V$.
 
 Under a perturbation budget $D$ such that $D\geq \sum_{i=1}^{k} \sigma_{i}$, the smallest achievable mutual information  is 0, and an optimal perturbation is taking  
 $E^{*}=-HU.$
\end{thm}

\begin{proof}
(of Theorem \ref{Thm:GaussianMultipleCopies}) Without loss of generality, we assume $n\leq m$, and we will show Theorem \ref{Thm:GaussianMultipleCopies} for $U_i\sim\mcN(0,\sigma^2)$ . Then $U':=V^TU$ performs a (projected) rotation of $U$ via matrix $V$, and each element of $U':=V^TU$ follows the same distribution as elements in $U$, i.e., $U'\sim \mcN(0,\sigma^2I)$ where $I$ is an identity matrix of dimension $k \times k$. Moreover,  $U'':=CV^TU$ is a vector following the $k$ dimensional joint Gaussian distribution, i.e., $U''=[\sigma_1 U'_1, \sigma_2 U'_2, \cdots, \sigma_k U'_k]^T\sim \mcN(0,G_{U''})$ where
\begin{align*}
G(U'') = \left[\begin{matrix}
\sigma_1^2\sigma^2 &0 &\cdots & 0\\
0 & \sigma_2^2\sigma^2 &\cdots & 0\\
\vdots &\vdots &\ddots &\vdots\\
0 &0 & \cdots & \sigma_k^2\sigma^2\\
\end{matrix}\right].
\end{align*}

To minimize the mutual information between $U$ and $HU+E$,  we need to minimize the mutual information between $U''$ and $U''+Q^{T} E$. The constraint that $\mbE[\|E\|^{2}]\leq D $
translates to the constraint that $\mbE[Q^{T} E] \leq D$. From Theorem 10.3.3 in \cite{cover_elements_2012}, we know the smallest possible mutual information  $I(U''; U''+Q^{T} E)$ is given by 
\begin{align}
I(U;HU+E) = I(U''; U''+Q^{T} E)= \sum_{i=1}^k \frac{1}{2}\log\frac{\sigma_i^2\sigma^2}{D_i},
\end{align}
where
\begin{align}
D_i =
\begin{cases}
\tau, \tau<\sigma_i^2\sigma^2,\\
\sigma_i^2\sigma^2, \tau\geq \sigma_i^2\sigma^2,
\end{cases}
\end{align}
and $\tau$ is chosen such that $\sum_{i=1}^k D_i=D$. We can first design the optimal adversarial attack $\Lambda\in\mbR^k$ for $U''$, and then perform the same rotation $Q$ (as $U''$ goes through) to get the $E=Q\Lambda$. From \cite{cover_elements_2012}, we know the optimal $\Lambda$ should follow $\mcN(\bm{0},G_\Lambda)$, thus the optimal $E=Q\Lambda$. By taking $\sigma^2=1$, we get Theorem \ref{Thm:GaussianMultipleCopies}.
\end{proof}

Theorem \ref{Thm:GaussianMultipleCopies} gives the minimal mutual information that can be achieved by an adversarial perturbation $E$ under the power constraint, and also characterizes the optimal perturbation an adversary can find. The minimal mutual information measures the amount of damage the adversary can make to the decision of the system, e.g., reducing the confidence of classification task held by deep learning agents or even fool them to make totally wrong classifications. The assignment of distortion budget $D_i$ for adversarial perturbation distribution $E$ has similar flavor as the water-filling interpretation in information theory \cite{cover_elements_2012}. However, we use a linear transformation for source message $U$ in hope that the system decision will be more difficult to be compromised by the adversaries.

\subsubsection{Falling Bar Algorithm}

We now propose a simple algorithm, which we call ``falling bar'' algorithm, for determining $\tau$ and $D_i$ in Algorithm \ref{Alg:JCAlgorithm}.  We set the initial bar $\tau$ to be very big so that all the values $\sigma_i^2$ is below the theshold. If the total quantity below the threshold is too big and larger than the capacity $D$, then we decrease the bar $\tau$ so that less quantity is smaller than the . This process repeats until we find the appropriate bar $\tau$. Assume $\sigma_i$ is the $i$-th largest singular values of $H$, and we define an interval
\begin{align}
	R_i =
	\begin{cases}
	[\sigma_{i+1}^2, +\infty), i=0\\
	[\sigma_{i+1}^2, \sigma_i^2), 1\leq i \leq k-1,\\
	(-\infty,\sigma_k^2), i=k.
	\end{cases}
\end{align}

\begin{algorithm} 
	\caption{Falling bar algorithm for determining $\tau$ and $D_i$.} 
	\label{Alg:JCAlgorithm} 
	\begin{algorithmic}[1] 
		\STATE Input: distribution of $U_i\sim\mcN)(0,1)$, $H\in\mbR^{n\times m}$ and $D\in[0,\|H\|_F^2]$.
		
		\STATE Output: $D_i, i=1,\cdots,k$ and $\tau$.
		
		\STATE Compute singular values of $H$, i.e., $\sigma_i, i=1,\cdots,k$ and construct intervals $R_i, i=0,\cdots,k$.
		
	    \FOR{$i=0,1,\cdots,k$}
	
	        \IF{$i=0$}
	
	            \IF{$\sum_{i=1}^k \sigma_i^2 = D$}
	
	                \STATE Choose $\tau$ to be any value greater than $\sigma_{1}^2$, and $D_i=\sigma_i^2, i=1,\cdots,k$
	
	                \STATE Break
	
	            \ENDIF
	
	        \ELSE
	
	             \STATE Compute $\tau = \frac{D - \sum_{j=i+1}^k \sigma_j^2}{i}$
	
	             \IF{$\tau\in R_i$}
	
	                 \STATE Compute $D_j=
	                 	\begin{cases}
	                 	\tau, j = 1,2,\cdots,i,\\
	                 	\sigma_j^2, j = i+1, i+2, \cdots,k
	                 	\end{cases}$
	
	                 \STATE Break
	
	             \ENDIF
	
	        \ENDIF
	
	    \ENDFOR
	
	    \STATE Return $\tau$ and $D_i, i=1,\cdots,k$.
	\end{algorithmic}
\end{algorithm}

\subsection{Adversarial Attacks of Binary Symmetric Channel}\label{Sec:AdversarialAttackBSC}

We now consider adversarial attacks on information communicated via discrete random variable, which has similar flavor as the binary symmetric channel in communication. The joint distribution $p(U,X)$ of $(U,X)$ is defined over $\{0,1\}\times\{0,1\}$, and we denote the probabilities $\mbP(U=0,X=0), \mbP(U=0,X=1), \mbP(U=1,X=0), \mbP(U=1,X=1)$ by $a,b,c$, and $d$, respectively. The distribution $p(E|U,X)$ of $E$ conditioning on $(U,X)$ is defined over $\{0,1\}$, and we denote by $p_1, p_2, p_3$, and $p_4$ the following probabilities respectively
\begin{align}\label{Defn:EConditionUX}
\begin{cases}
\mbP(Y=1|U=0,X=0) = \mbP(E=1|U=0,X=0), \\
\mbP(Y=0|U=0,X=1) = \mbP(E=1|U=0,X=1), \\
\mbP(Y=1|U=1,X=0) = \mbP(E=1|U=1,X=0),\\
\mbP(Y=0|U=1,X=1) = \mbP(E=1|U=1,X=1),
\end{cases}
\end{align}
where we use the definition $Y=:X+E$, and the addition is modular over 2. The overall scheme is illustrated in Figure \ref{Fig:BinarySymmetricChannel}.

\begin{figure}[!htb]
\centering
\includegraphics{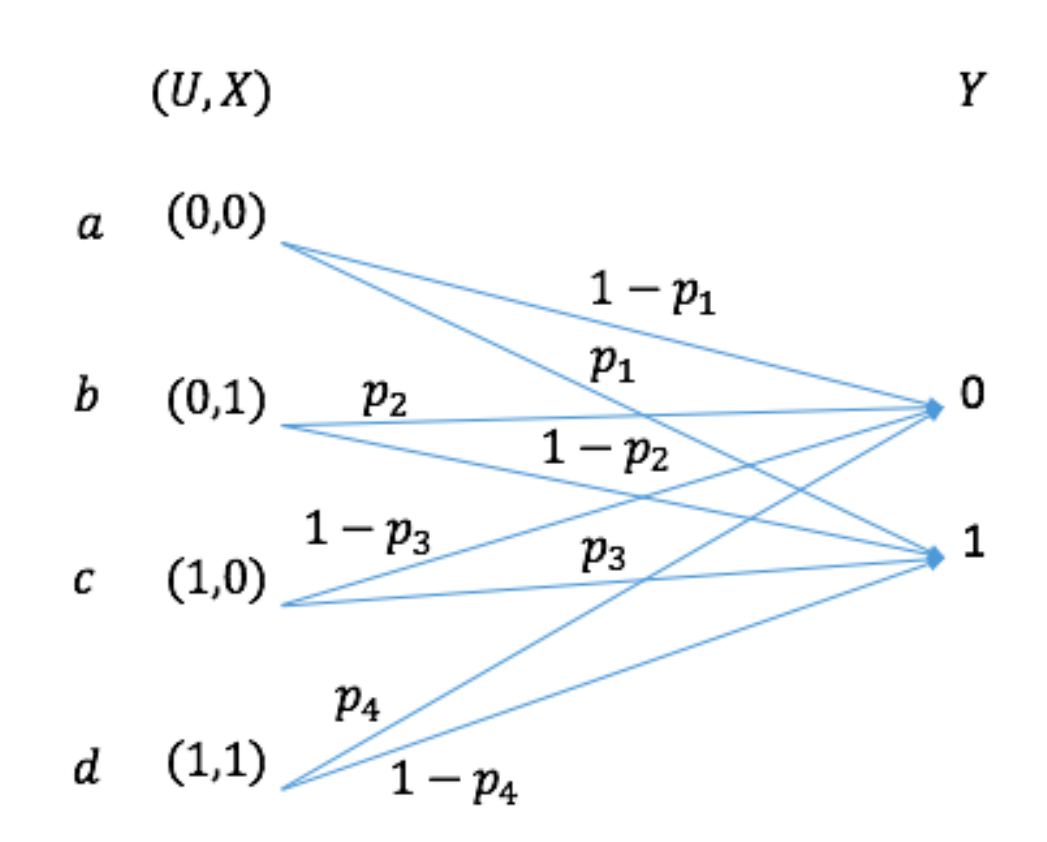}
\caption{Illustration of binary symmetric channel: the source $U$ and the synthesized $X$ are jointly distributed.}\label{Fig:BinarySymmetricChannel}
\end{figure} 

From Figure \ref{Fig:BinarySymmetricChannel}, the joint distribution of $(U,Y)$ or $(U,X+E)$ is
\begin{align}\label{Eq:JointUY}
\begin{cases}
\mbP(U=0,Y=1) = ap_1+b(1-p_2), \\
\mbP(U=0,Y=0) = a(1-p_1) + bp_2, \\
\mbP(U=1,Y=0) = c(1-p_3)+dp_4, \\
\mbP(U=1,Y=1) = cp_3+d(1-p_4),
\end{cases}
\end{align}
the marginal distribution $p(U)$ of $U$ is
\begin{align}\label{Eq:MarginalU}
\mbP(U = 0) = a+b,
\mbP(U = 1) = c+d,
\end{align}
and the marginal distribution $p(Y)$ of $Y$ is
\begin{align}\label{Eq:MarginalY}
\mbP(Y=0) = a(1-p_1) + bp_2 +c(1-p_3) + d p_4,
\mbP(Y=1) = ap_1 + b(1-p_2) + cp_3 + d(1-p_4).
\end{align}
We will denote $\mbP(Y=0)$ by $A$. From (\ref{Eq:JointUY}), (\ref{Eq:MarginalU}) and (\ref{Eq:MarginalY}), we can compute the mutual information as follows
\begin{align}\label{Eq:BSCMutualInfo}
I(U;Y)
& = H(Y) - H(Y|U)\nonumber\\
& = H(Y) - \mbP(U=0) H(Y|U=0) - \mbP(U=1)H(Y|U=1)\nonumber\\
& = \left(a(1-p_1) + bp_2 +c(1-p_3) + d p_4\right) \log\left(\frac{1}{a(1-p_1) + bp_2 +c(1-p_3) + d p_4}\right) \nonumber\\
&\quad + \left(ap_1 + b(1-p_2) + cp_3 + d(1-p_4)\right) \log\left(\frac{1}{ap_1 + b(1-p_2) + cp_3 + d(1-p_4)}\right) \nonumber\\
& \quad - (a+b) \left[\frac{a(1-p_1) + bp_2}{a+b} \log \frac{a+b}{a(1-p_1) + bp_2} + \frac{ap_1+b(1-p_2)}{a+b} \log\frac{a+b}{ap_1+b(1-p_2)}\right] \nonumber\\
& \quad - (c+d) \left[\frac{c(1-p_3) + dp_4}{c+d} \log\frac{c+d}{c(1-p_3) + dp_4}
+ \frac{cp_3+ d(1-p_4)}{c+d} \log\frac{c+d}{cp_3+d(1-p_4)}\right] \nonumber\\
& = I_1(p_1,p_2,p_3,p_4) + I_2(p_1,p_2,p_3,p_4)
-(a+b)I_3(p_1,p_2) - (c+d)I_4(p_3,p_4),
\end{align}
where we use the conditional distribution $p(Y|U)$, i.e.,
\begin{align}\label{Eq:YConditionU}
\begin{cases}
\mbP(Y=0|U=0) = \frac{a(1-p_1) + bp_2}{a+b},\\
\mbP(Y=1|U=0) = \frac{ap_1 + b(1-p_2)}{a+b},\\
\mbP(Y=0|U=1) = \frac{c(1-p_3) + dp_4}{c+d},\\
\mbP(Y=1|U=1) = \frac{cp_3 + d(1-p_4)}{c+d},
\end{cases}
\end{align}
and define
\begin{align*}
\begin{cases}
I_1(p_1,p_2,p_3,p_4) := \left(a(1-p_1) + bp_2 +c(1-p_3) + d p_4\right) \log\left(\frac{1}{a(1-p_1) + bp_2 +c(1-p_3) + d p_4}\right),\\
I_2(p_1,p_2,p_3,p_4) := \left(ap_1 + b(1-p_2) + cp_3 + d(1-p_4)\right) \log\left(\frac{1}{ap_1 + b(1-p_2) + cp_3 + d(1-p_4)}\right),\\
I_3(p_1,p_2) := \left[\frac{a(1-p_1) + bp_2}{a+b} \log \frac{a+b}{a(1-p_1) + bp_2} + \frac{ap_1+b(1-p_2)}{a+b} \log\frac{a+b}{ap_1+b(1-p_2)}\right],\\
I_4(p_3,p_4) := \left[\frac{c(1-p_3) + dp_4}{c+d} \log\frac{c+d}{c(1-p_3) + dp_4}
+ \frac{cp_3+ d(1-p_4)}{c+d} \log\frac{c+d}{cp_3+d(1-p_4)}\right].
\end{cases}
\end{align*}
Thus, by using natural base for the logarithmic function, we have
\begin{align}\label{Eq:DerivativeComponentsP1}
\frac{\partial I_1(p_1,p_3,p_3,p_4)}{\partial p}
= \left[\begin{matrix}
(-a)\log\left(\frac{1}{a(1-p_1)+bp_2+c(1-p_3)+dp_4}\right)  + (-1)(-a)\\
b\log\left(\frac{1}{a(1-p_1)+bp_2+c(1-p_3)+dp_4}\right) + (-1)b\\
(-c)\log\left(\frac{1}{a(1-p_1)+bp_2+c(1-p_3)+dp_4}\right) + (-1)(-c)\\
d\log\left(\frac{1}{a(1-p_1)+bp_2+c(1-p_3)+dp_4}\right) + (-1)d
\end{matrix}\right],
\end{align}
\begin{align}\label{Eq:DerivativeComponentsP2}
\frac{\partial I_2(p_1,p_3,p_3,p_4)}{\partial p}
= \left[\begin{matrix}
a\log\left(\frac{1}{ap_1+b(1-p_2)+cp_3+d(1-p_4)}\right)  + (-1)a\\
(-b)\log\left(\frac{1}{ap_1+b(1-p_2)+cp_3+d(1-p_4)}\right) + (-1)(-b)\\
c\log\left(\frac{1}{ap_1+b(1-p_2)+cp_3+d(1-p_4)}\right) + (-1)c\\
(-d)\log\left(\frac{1}{ap_1+b(1-p_2)+cp_3+d(1-p_4)}\right) + (-1)(-d)
\end{matrix}\right],
\end{align}
\begin{align}\label{Eq:DerivativeComponentsP3}
\frac{\partial I_3(p_1,p_2)}{\partial p}
= \left[\begin{matrix}
\frac{a}{a+b}\log\frac{\mbP(Y=0|U=0)}{\mbP(Y=1|U=0)}\\
\frac{b}{a+b}\log\frac{\mbP(Y=1|U=0)}{\mbP(Y=0|U=0)}\\
0 \\
0
\end{matrix}\right],
\end{align}
and
\begin{align}\label{Eq:DerivativeComponentsP4}
\frac{\partial I_3(p_1,p_2)}{\partial p}
= \left[\begin{matrix}
0 \\ 0 \\
\frac{c}{a+b}\log\frac{\mbP(Y=0|U=1)}{\mbP(Y=1|U=1)}\\
\frac{d}{a+b}\log\frac{\mbP(Y=1|U=1)}{\mbP(Y=0|U=1)}
\end{matrix}\right].
\end{align}

To perform adversarial attack, we propose to minimize $I(U;Y)$ in (\ref{Eq:BSCMutualInfo}) over $p(E|U,X)$ under distortion constraint, i.e.,
\begin{align}\label{Defn:BSCAdversarialAttackExplicit}
\min_{p_1,p_2,p_3,p_4\in[0,1]} I(U;Y), {\rm s.t.\ }ap_1+bp_2+cp_3+dp_4 \leq \epsilon.
\end{align}

From Section \ref{sec:problemformulation}, we know \eqref{Defn:BSCAdversarialAttackExplicit} is convex. In Theorem \ref{Thm:OptimalityConditions}, we give the necessary optimality conditions for the optimal solution to \eqref{Defn:BSCAdversarialAttackExplicit}.

\begin{thm}\label{Thm:OptimalityConditions}
Let $p^*=[p_1^*,p_2^*,p_3^*,p_4^*]^T\in[0,1]^4$ be the optimal solution of problem \eqref{Defn:BSCAdversarialAttackExplicit}, then there must exist a $\lambda^*\geq0$ such that
\begin{align}\label{Eq:PrimalFeas}
ap_1^*+bp_2^*+cp_3^*+dp_4^*\leq \epsilon,
\end{align}
\begin{align}\label{Eq:DualFeas}
\begin{cases}
a\left(\log\frac{\mbP^*(Y=0)\mbP^*(Y=1|U=0)}{\mbP^*(Y=1)\mbP^*(Y=0|U=0)} + \lambda^*\right) = 0,\\
b\left(\log\frac{\mbP^*(Y=1)\mbP^*(Y=0|U=0)}{\mbP^*(Y=0)\mbP^*(Y=1|U=0)} + \lambda^* \right) = 0,\\
c\left(\log\frac{\mbP^*(Y=0)\mbP^*(Y=1|U=1)}{\mbP^*(Y=1)\mbP^*(Y=0|U=1)} + \lambda^*\right) = 0, \\
d\left(\log\frac{\mbP^*(Y=1)\mbP^*(Y=0|U=1)}{\mbP^*(Y=0)\mbP^*(Y=1|U=1)} + \lambda^* \right) = 0,
\end{cases}
\end{align}
and
\begin{align}\label{Eq:ComplementarySlack}
\lambda^*(ap_1^*+bp_2^*+cp_3^*+dp_4^* - \epsilon) = 0,
\end{align}
where the probability $\mbP^*(\cdot)$ is computed using $p^*$. Furthermore, if both $a$ and $b$ are nonzero, or both $c$ and $d$ are zero, then $\lambda^*=0$. Moreover, if none of $a,b,c,d$ is zero, we have
\begin{align}\label{Eq:SpecialEquality}
\frac{\mbP^*(Y=0|U=0)}{\mbP^*(Y=1|U=0)} = \frac{\mbP^*(Y=0|U=1)}{\mbP^*(Y=1|U=1)} = \frac{\mbP^*(Y=0)}{\mbP^*(Y=1)}.
\end{align}
\end{thm}

\begin{proof}
(of Theorem \ref{Thm:OptimalityConditions}) We can get the Lagrangian of \eqref{Defn:BSCAdversarialAttackExplicit} as
\begin{align}\label{Eq:Lagrangian}
\mcL(p,\lambda) = I(U;Y) + \lambda(ap_1+bp_2+cp_3+dp_4 - \epsilon),
\end{align}
and from \eqref{Eq:DerivativeComponentsP1}, \eqref{Eq:DerivativeComponentsP2}, \eqref{Eq:DerivativeComponentsP3}, and \eqref{Eq:DerivativeComponentsP4}, we have the partial derivatives as
\begin{align}
\begin{cases}
\frac{\partial \mcL(p,\lambda)}{\partial p_1}
= a\left(\log\frac{\mbP(Y=0)\mbP(Y=1|U=0)}{\mbP(Y=1)\mbP(Y=0|U=0)} + \lambda\right),\\
\frac{\partial \mcL(p,\lambda)}{\partial p_2}
= b\left(\log\frac{\mbP(Y=1)\mbP(Y=0|U=0)}{\mbP(Y=0)\mbP(Y=1|U=0)} + \lambda \right),\\
\frac{\partial \mcL(p,\lambda)}{\partial p_3}
= c\left(\log\frac{\mbP(Y=0)\mbP(Y=1|U=1)}{\mbP(Y=1)\mbP(Y=0|U=1)} + \lambda\right), \\
\frac{\partial \mcL(p,\lambda)}{\partial p_4}
= d\left(\log\frac{\mbP(Y=1)\mbP(Y=0|U=1)}{\mbP(Y=0)\mbP(Y=1|U=1)} + \lambda \right).
\end{cases}
\end{align}
From the KKT conditions \cite{boyd_convex_2004}, we have for an optimal solution $p^*\in[0,1]^4$, there must exist a $\lambda^*\geq 0$ such that
\begin{align}
ap_1^*+bp_2^*+cp_3^*+dp_4^*\leq \epsilon,
\lambda^*(ap_1^*+bp_2^*+cp_3^*+dp_4^* - \epsilon) = 0
\end{align}
and
\begin{align}
\frac{\partial \mcL(p^*,\lambda^*)}{\partial p_1}
= 0,
\frac{\partial \mcL(p^*,\lambda^*)}{\partial p_2}
= 0,
\frac{\partial \mcL(p^*,\lambda^*)}{\partial p_3}
= 0,
\frac{\partial \mcL(p^*,\lambda^*)}{\partial p_4}
= 0.
\end{align}
Thus, we proved the first part of Theorem \ref{Thm:OptimalityConditions}.

When both $a$ and $b$ are nonzero, we have from \eqref{Eq:DualFeas}
\begin{align}
\log\frac{\mbP^*(Y=1|U=0)}{\mbP^*(Y=0|U=0)} = -\lambda^* - \log\frac{\mbP^*(Y=0)}{\mbP^*(Y=1)},
\log\frac{\mbP^*(Y=0|U=0)}{\mbP^*(Y=1|U=0)} = -\lambda^* + \log\frac{\mbP^*(Y=0)}{\mbP^*(Y=1)},
\end{align}
or
\begin{align}
\log\frac{\mbP^*(Y=1|U=0)}{\mbP^*(Y=0|U=0)} = -\lambda^* - \log\frac{\mbP^*(Y=0)}{\mbP^*(Y=1)},
\log\frac{\mbP^*(Y=1|U=0)}{\mbP^*(Y=0|U=0)} = \lambda^* - \log\frac{\mbP^*(Y=0)}{\mbP^*(Y=1)},
\end{align}
thus $\lambda^*=0$, and
\begin{align}
\frac{\mbP^*(Y=1|U=0)}{\mbP^*(Y=0|U=0)} = \frac{\mbP^*(Y=1)}{\mbP^*(Y=0)}.
\end{align}
Similarly for the case where both $c$ and $d$ are nonzero, we can get
\begin{align}
\frac{\mbP^*(Y=1|U=1)}{\mbP^*(Y=0|U=1)} = \frac{\mbP^*(Y=1)}{\mbP^*(Y=0)}.
\end{align}
When none of $a,b,c,d$ is zero, we get the above two equations, and thus \eqref{Eq:SpecialEquality}.
\end{proof}

Theorem \ref{Thm:OptimalityConditions} characterizes the distribution of $Y$ that is achieved by the optimal adversarial perturbation $E$. Theorem \ref{Thm:OptimalityConditions} implies that when $D$ is large enough, the optimal $E$ should achieve a distribution for $Y$ which is independent of the source $U$, e.g., $\mbP^*(Y=0)=\mbP^*(Y=0|U=0)=\mbP^*(Y=1|U=0)$. This means a zero mutual information, which coincides with the intuition in rate distortion theory, i.e., a big enough distortion $D$ can lead to zero rate.

\section{Adversarial Attacks Over Subset of Random Variables}\label{Sec:L1MetricAttack}

 In this section, we consider the problem of information-theoretically optimal adversarial attacks over subset of random variables where the attacker can attack only a given small number of outputs. For this problem, we will derive the optimal attacking strategies to minimize the mutual information between $U\in\mbR^m$ and $Y=X+E=U+E\in\mbR^m$ under sparse attacks over a subset of random variables indexed by $T\subset\{1,2,\cdots,m\}$ and $|T|=k<m$.

 We assume $(U,X)\sim p(U,X)$, and the attacker picks a fixed subset of random variables $X_T\in\mbR^{|T|}$ which consists of elements from $X$ as indexed by $T$, and the attacker can arbitrarily change the output of them by designing $E_T\in\mbR^{|T|}$ or $p(E_T|U,X)$. This can be formulated as follows
\begin{align}\label{Defn:AttackingModelOverSubsetRandomVariables}
&\min_{T\subset[m],p(E_T|U,X)} ~~~~~~I(U;U+E)\nonumber\\
&\text{s.t.}~~~~~~  E_{\overline{T}}=0,
\end{align}
where $[m]$ is defined as $\{1,2,\cdots,m\}$, $\overline{T}$ is the complement of the set of attacked sensors, and $p(E_T|U,X)$ is an arbitrary valid distribution. Under this formulation, we have characterizations of the optimal attack as stated in Theorem \ref{Thm:OptimalAttackOverSubsetofRandomVariables}.

\begin{thm}\label{Thm:OptimalAttackOverSubsetofRandomVariables}
Let $U\in\mbR^m$ follow Gaussian distribution $\mcN(0,\Sigma)$ where
\begin{align*}
\Sigma = \left[\begin{matrix}
\sigma_1^2 &0 &\cdots &0\\
0 &\sigma_2^2 &\cdots &0\\
\vdots & \vdots &\ddots &\vdots\\
0 &0 &\cdots &\sigma_m^2.
\end{matrix}\right]
\end{align*}
Under the attack formulation \eqref{Defn:AttackingModelOverSubsetRandomVariables}, the minimum mutual information will be
\begin{align}
I(U;U+E^*) = \frac{1}{2}\log\left((2\pi e)^{|\overline{T^*}|}|\Sigma_{\overline{T^*}}|\right),
\end{align}
where $T^*$ is the index set such that the submatrix of $\Sigma$ with rows specified by $\overline{T^*}$ and columns specified by $\overline{T^*}$ has the smallest determinant.
\end{thm}

\begin{proof}
(of Theorem \ref{Thm:OptimalAttackOverSubsetofRandomVariables}) From the properties of mutual information and different entropy, we have
\begin{align}\label{Eq:MIWithoutSplitting}
I(U;U+E)
& = h(U) - h(U;U+E)\\
& = \sum_{i=1}^m h(U_i|U_1,U_2,\cdots,U_{i-1}) - \sum_{i=1}^m h(U_i|U+E,U_1,U_2,\cdots,U_{i-1}) \\
& = I_1 + I_2,
\end{align}
where we define $I_1$ and $I_2$ as
\begin{align*}
I_1 = \sum_{i=T} (h(U_i|U_1,U_2,\cdots,U_{i-1})-h(U_i|U+E,U_1,U_2,\cdots,U_{i-1}))
\end{align*}
and
\begin{align*}
I_2 = \sum_{i=\overline{T}} (h(U_i|U_1,U_2,\cdots,U_{i-1})-h(U_i|U+E,U_1,U_2,\cdots,U_{i-1})).
\end{align*}

Without loss of generality, we assume the index set $T$ to be $T:=\{i_1,i_2,\cdots,i_k\}\subset[m]$ with $i_1<i_2<\cdots<i_k$. Thus, $\overline{T}=[m]\setminus T:=\{s_1,s_2,\cdots,s_{m-k}\}$ with $s_1<s_2<\cdots<s_{m-k}$. Since
\begin{align}
I_1
& \geq \sum_{j\in[k]} (h(U_{i_j}|U_{i_1},U_{i_2},\cdots,U_{i_{j-1}}) - h(U_{i_j}|U_T+E_T,U_{i_1},U_{i_2},\cdots,U_{i_{j-1}})) \nonumber\\
& = \sum_{j\in[k]} h(U_{i_j}|U_{i_1},U_{i_2},\cdots,U_{i_{j-1}}) - \sum_{j\in[k]} h(U_{i_j}|U_T+E_T,U_{i_1},U_{i_2},\cdots,U_{i_{j-1}})) \nonumber\\
& = h(U_T) - h(U_T|U_T+E_T)\nonumber\\
& = I(U_T;U_T+E_T)
\end{align}
where the first inequality holds due to the independence of elements of $U$ and the fact that conditioning reduces the differential entropy. From Theorem \ref{Thm:SingleGaussianVar}, $I_1$ can achieve 0 by appropriately choosing $p(E_T|U_T)$. Note that $I(U_T;U_T+E_T)$ actually corresponds to the vector form of that in Theorem \ref{Thm:SingleGaussianVar} without considering noise $W$. Similarly, we have from \cite{cover_elements_2012}
\begin{align}
I_2
& \geq \sum_{j\in[m-k]} (h(U_{s_j}|U_{s_{1}},U_{s_{2}},\cdots,U_{s_{j-1}}) - h(U_{s_{j}}|U_{\overline{T}}+E_{\overline{T}},U_{s_1},U_{s_2},\cdots,U_{s_{j-1}})) \\
& = h(U_{\overline{T}}) - h(U_{\overline{T}}|U_{\overline{T}}+E_{\overline{T}}) \\
& = h(U_{\overline{T}}) \\
& = \frac{1}{2}\log\left((2\pi e)^{|\overline{T}|} |\Sigma_{\overline{T}}\right).
\end{align}
Thus
\begin{align*}
I(U;U+E)
= I_1+I_2
\geq \frac{1}{2}\log\left((2\pi e)^{|\overline{T}|} |\Sigma_{\overline{T}}\right),
\end{align*}
and the minimum mutual information is achieved when $\overline{T^*}$ is the index set which contains the smallest $m-k$ diagonal elements of $\Sigma$.
\end{proof}

Theorem \ref{Thm:OptimalAttackOverSubsetofRandomVariables} actually implies that when the attacker can only attack over a subset of the random variables with size $k$ arbitrarily, he/she should attack the $k$ random variables with the largest variances. This also coincides with the intuition in rate distortion theory.

\section{Experimental Results}\label{Sec:ExperimentalResults}

In this section, we provide empirical results to validate our theoretical claims. In the first set of experiments, we consider the adversarial attacks via mutual information minimization \eqref{Defn:OptimizationforGaussian} in the general Gaussian case where the source distribution is Gaussian with zero mean. We take the variance $a^2$ of $U\sim\mcN(0,a^2)$ to be different values, i.e., $a^2\in\{0.1,0.3,0.7,0.9\}$, and similarly for $\sigma^2$ of $W\sim\mcN(0,\sigma^2)$, i.e., $\sigma^2\in\{0.1,0.2,0.3,0.4,0.5,0.6,0.7,0.8,0.9,1\}$. For each pair of $a^2, \sigma^2$, we empirically show how the mutual information changes with the distortion $D$ when $D$ takes value from $\{0.1,0.2,0.3,0.4,0.5,0.6,0.7,0.8,0.9,1\}$. For each group of $(a^2,\sigma^2,D)$, we generate meshgrid for $x\in[-\sqrt{a^2D},\sqrt{a^2D}]$ and $y\in[-\sqrt{\sigma^2D},\sqrt{\sigma^2D}]$, and the get all the grid points satisfying the constraint. We then find from these grid points the poin which achieves the minimal mutual information. The results are presented in Figure \ref{Fig:GeneralGaussian}.

\begin{figure}[!htb]
	\centering
	\begin{subfigure}[b]{0.5\textwidth}
		\centering
		\includegraphics[width=\linewidth]{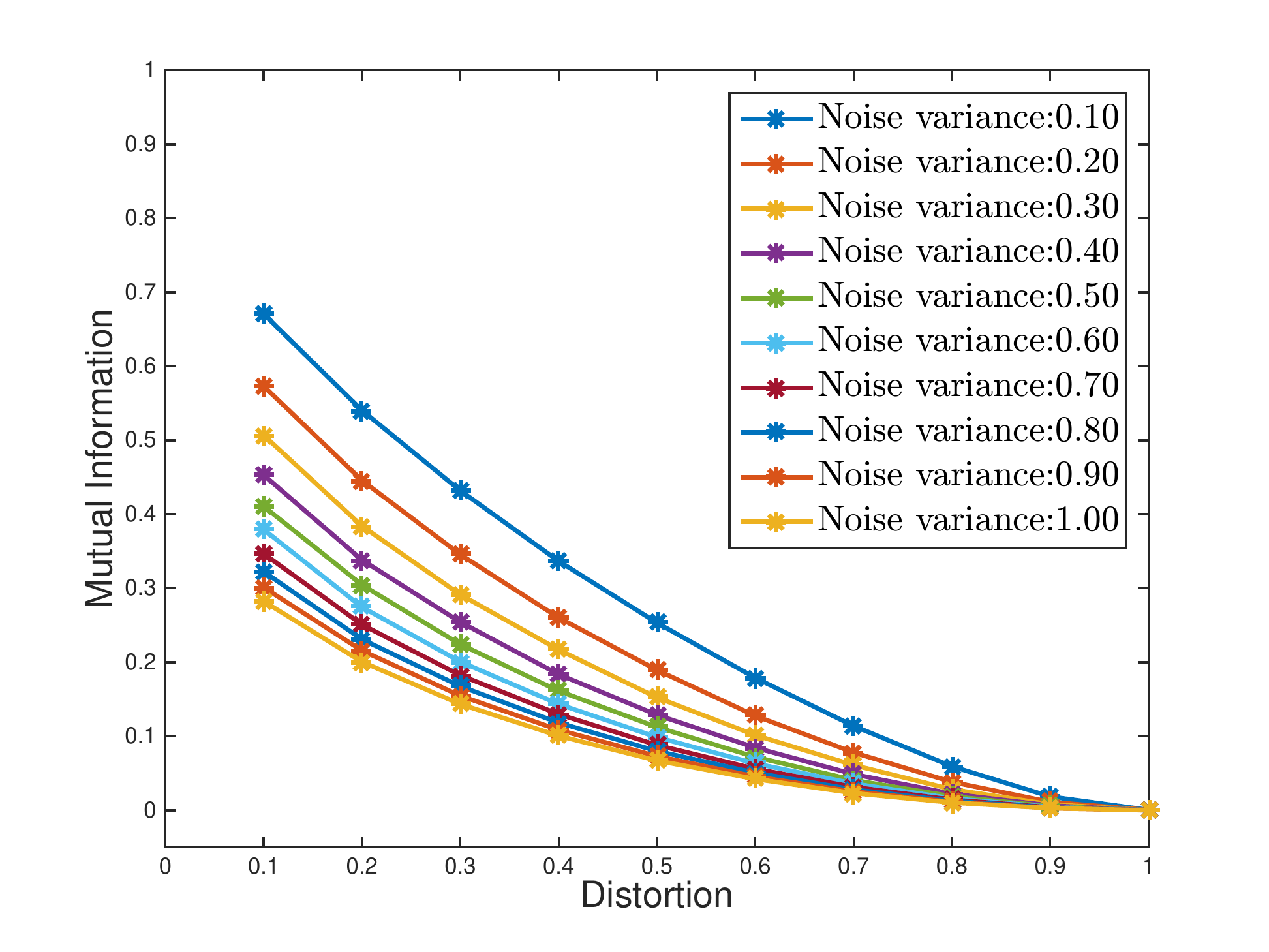}
		\caption{Variance $a^2$ of $U$: 1}
	\end{subfigure}%
	~
	\begin{subfigure}[b]{0.5\textwidth}
		\centering
		\includegraphics[width=\linewidth]{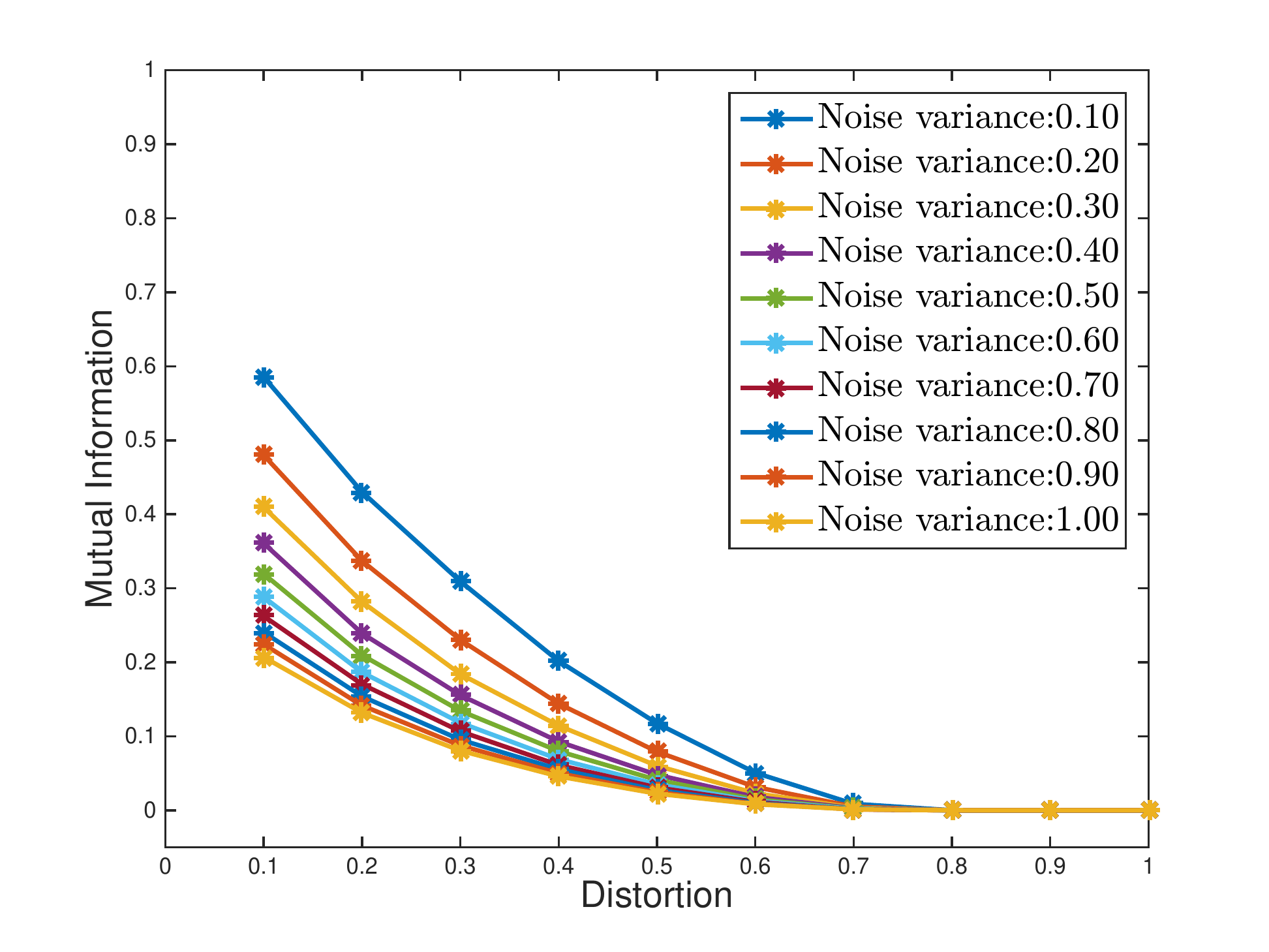}
		\caption{Variance $a^2$ of $U$: 0.75}
	\end{subfigure}
	\begin{subfigure}[b]{0.5\textwidth}
	\centering
	\includegraphics[width=\linewidth]{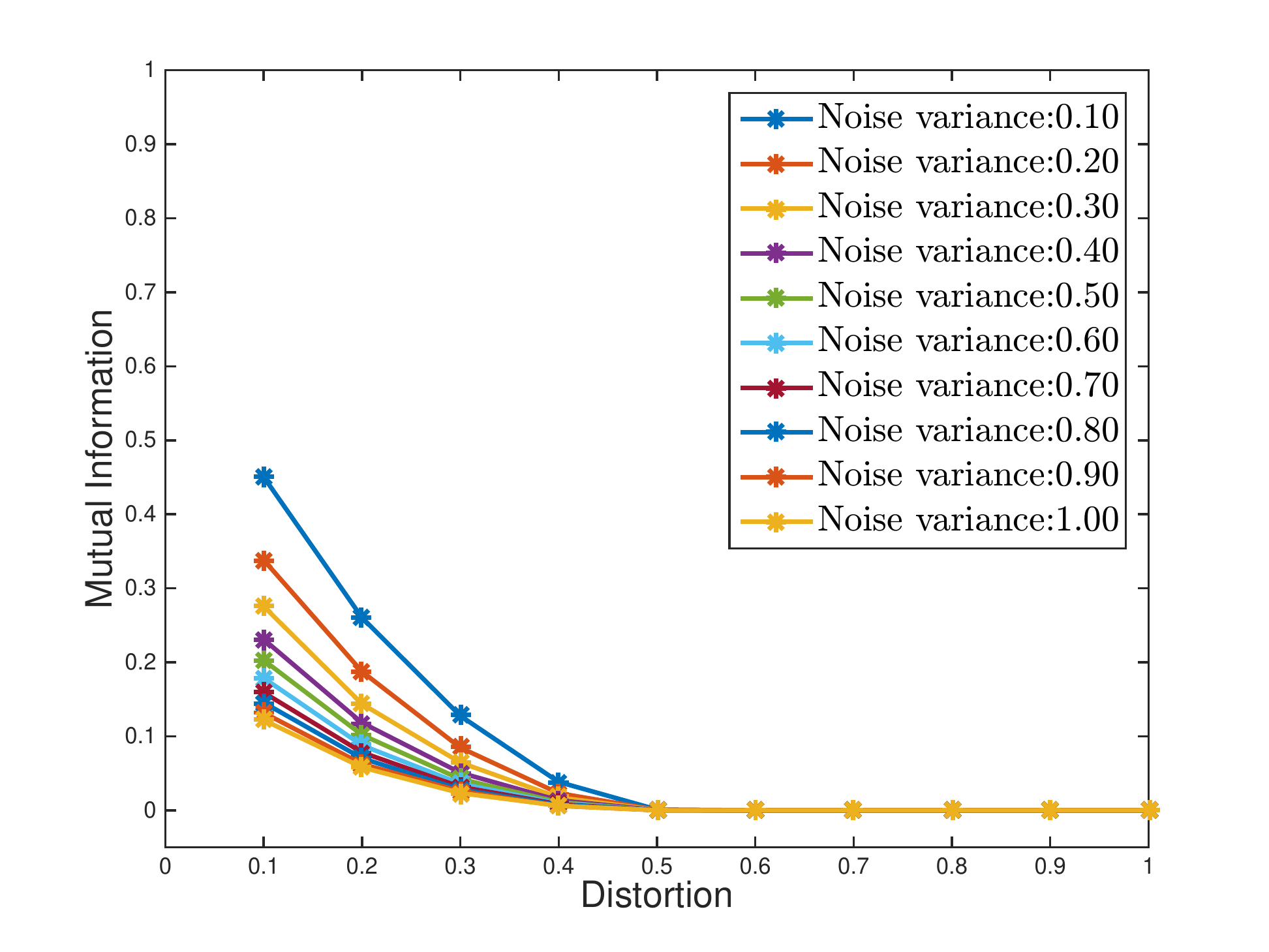}
	\caption{Variance of $a^2$ of $U$: 0.5}
\end{subfigure}%
~
\begin{subfigure}[b]{0.5\textwidth}
	\centering
	\includegraphics[width=\linewidth]{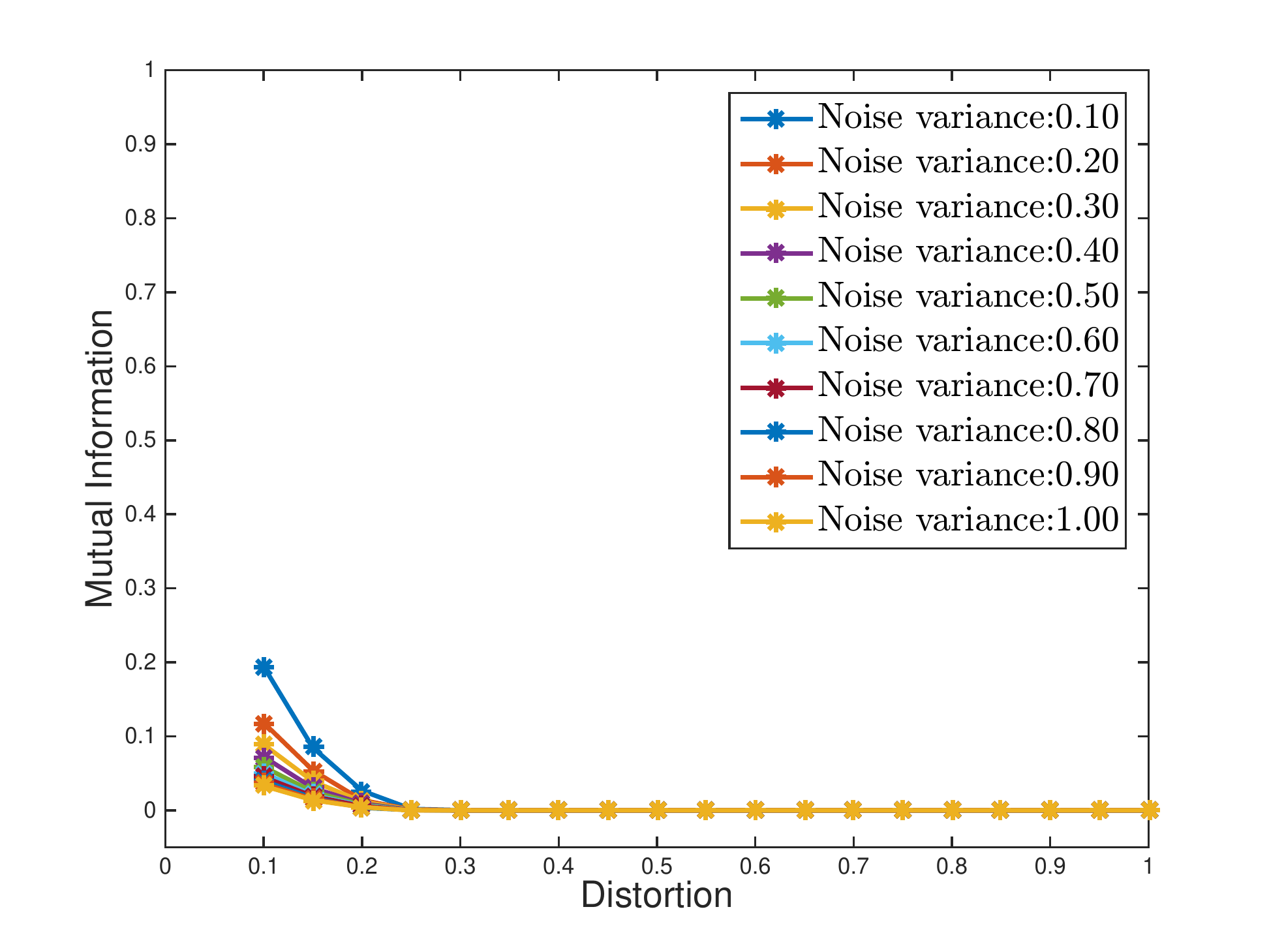}
	\caption{Variance $a^2$ of $U$: 0.25}
\end{subfigure}%
	\caption{Minimum mutual information under different distortion budgets.}\label{Fig:GeneralGaussian}
\end{figure}

From the results, we can see that for fixed source variance $a^2$ and random noise variance $\sigma^2$, the mutual information decreases monotonically with respect to distortion $D$. Besides, to achieve a zero mutual information, i.e., no information can be conveyed about label $U$, the magnitude of the distortion should be comparable to the source variance. For example, when $a^2=1$, the distortion needs to be larger than 0.9 so that an almost zero mutual information can be achieved regardless of the noise variance. Moreover, the adversarial perturbation $E$ can cause more damage to reduce the mutual information than the random noise $W$ does. For example, in the case where $a^2=0.75$ and $D=0.1$, the mutual information decreases from 0.6 to 0.2 as the noise variance increases from 0.1 to 1. However, in the case where $a^2=0.75$ and $\sigma^2=0.1$, the mutual information can decreases from 0.6 to 0 as distortion increases from 0.1 to 1.

In the second set of experiments, we demonstrate the performance of attacking the multiple Gaussian copies with linear projection based on Theorem \ref{Thm:GaussianMultipleCopies}. We take $m$ to be different values, i.e., $m\in\{10,50,250,1250\}$, and the $n$ will take $\alpha m$ with $\alpha$ taking different values, i.e., $\alpha\in\{0.1,0.2,\cdots,0.9\}$ for fat projection matrix, and $\alpha\in\{1.1,1.2,\cdots,1.9\}$ for tall projection matrix. For each pair of $(m,\alpha)$, the projection matrix $H\in\mbR^{n\times m}$ is generated randomly with elements i.i.d. according to $\mcN(0,\frac{1}{m})$. Then, we compute the minimal mutual information for a given $D$ which can take values in $\{0.1,0.2,0.3,0.4,0.5,0.6,0.7,0.8,0.9,1\}\times\sum_{i} \sigma_i^2$. The results are presented in Figure \ref{Fig:MultipleCopiesGaussian} and \ref{Fig:MultipleCopiesGaussianTall}.

\begin{figure}[!htb]
	\centering
	\begin{subfigure}[b]{0.5\textwidth}
		\centering
		\includegraphics[width=\linewidth]{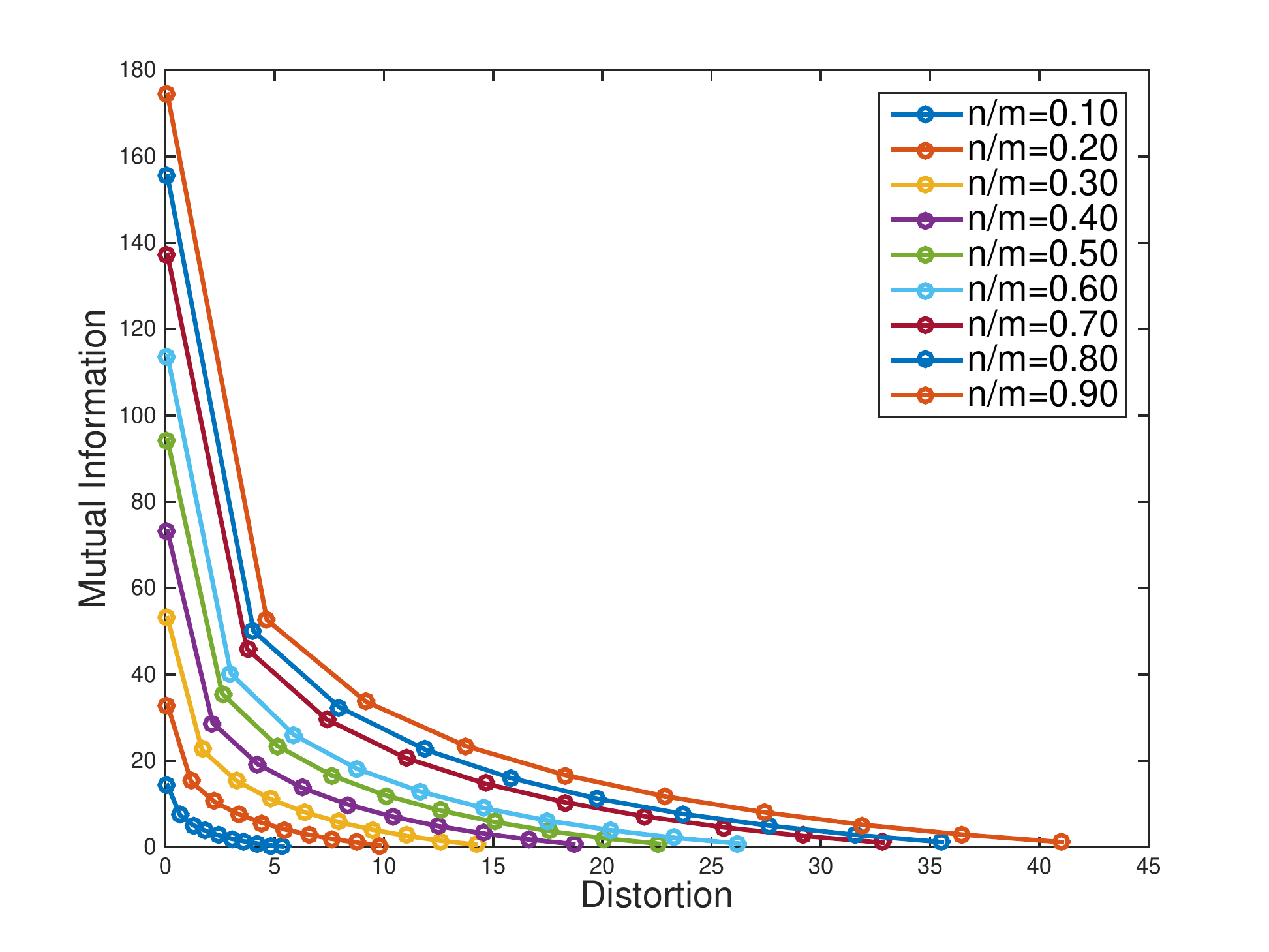}
		\caption{Dimension $m$ of $U$: 50}
	\end{subfigure}%
	~
	\begin{subfigure}[b]{0.5\textwidth}
		\centering
		\includegraphics[width=\linewidth]{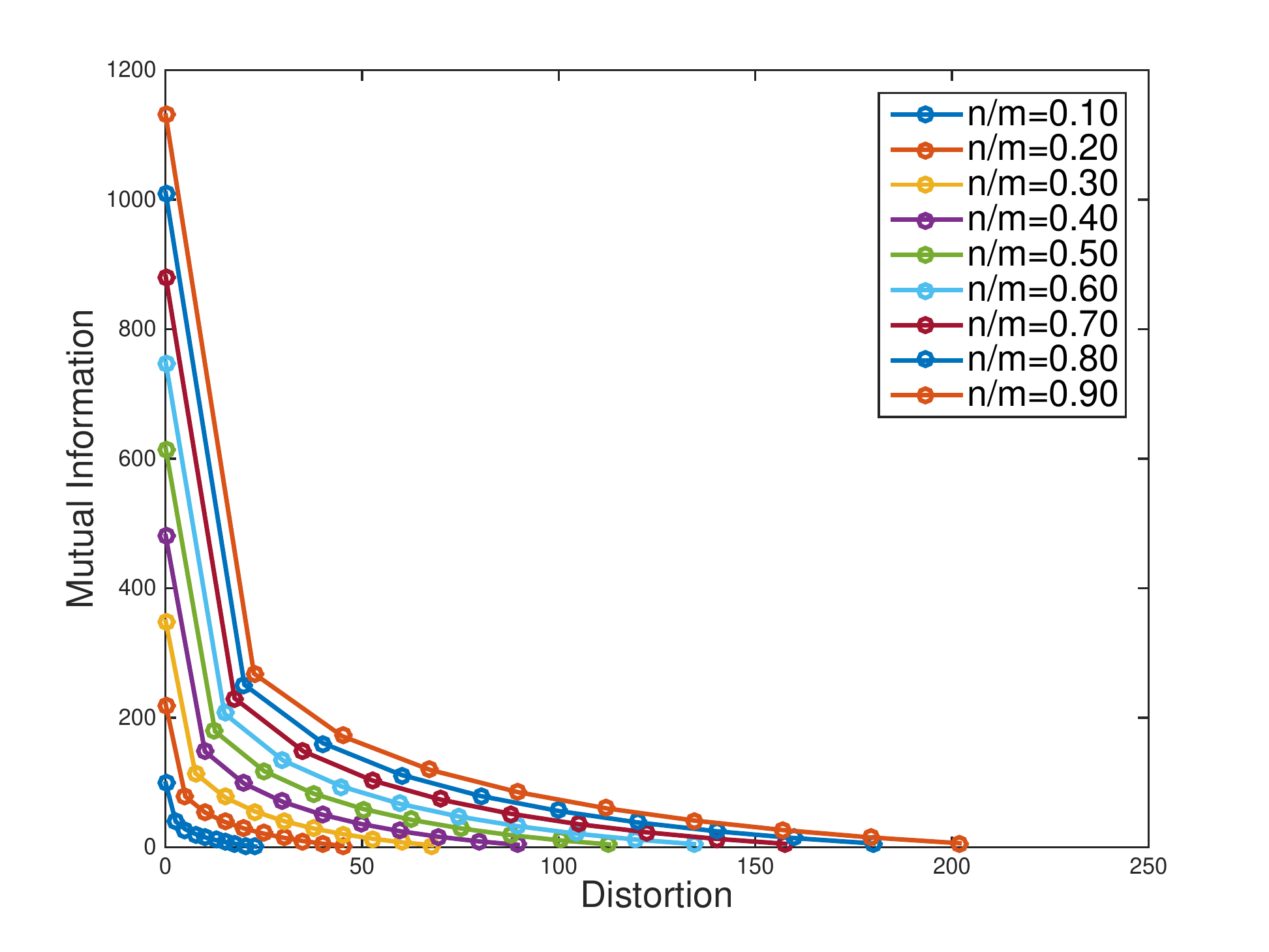}
		\caption{Dimension $m$ of $U$: 250}
	\end{subfigure}
	\begin{subfigure}[b]{0.5\textwidth}
	\centering
	\includegraphics[width=\linewidth]{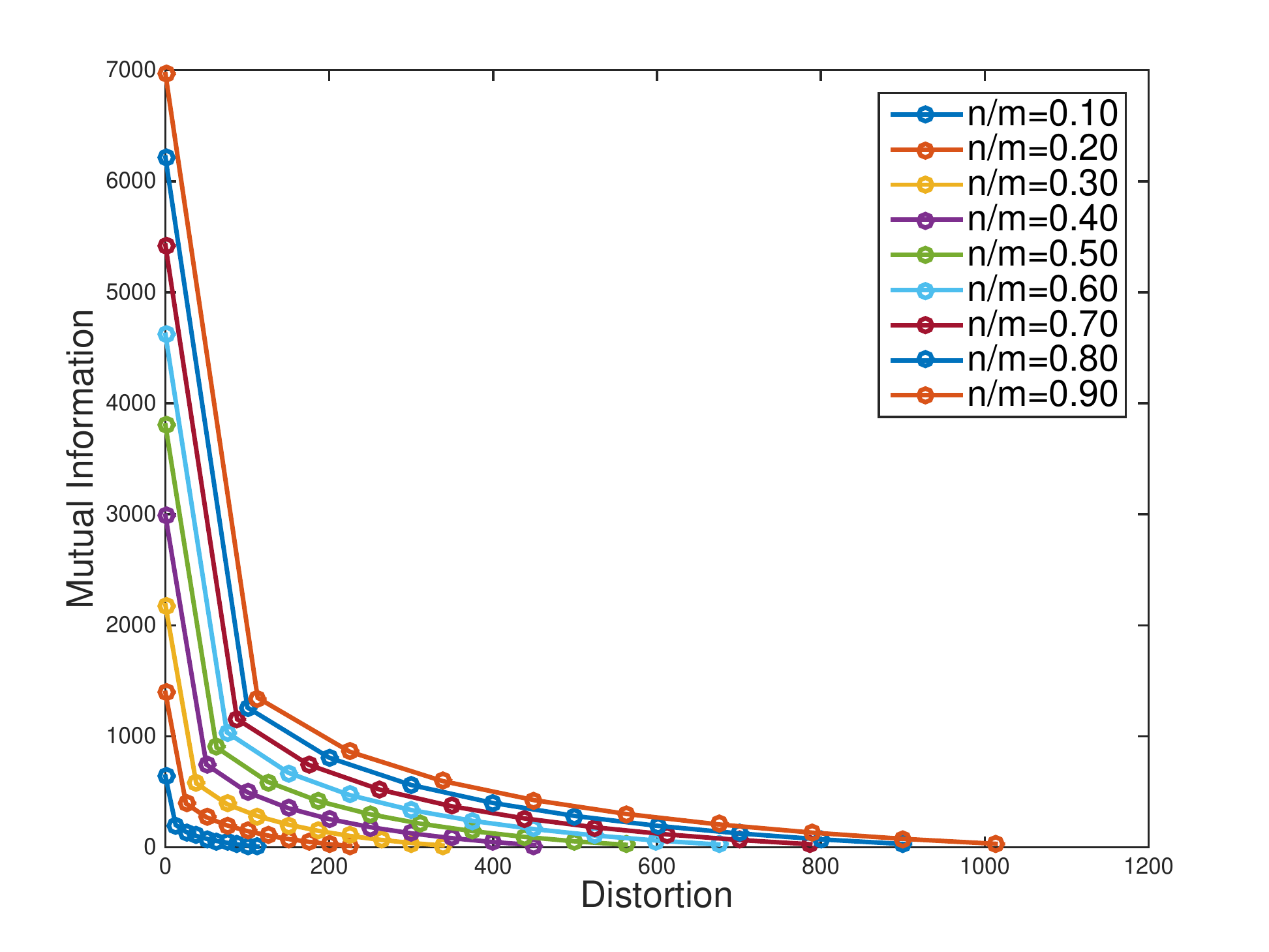}
	\caption{Dimension $m$ of $U$: 1250}
\end{subfigure}%
~
\begin{subfigure}[b]{0.5\textwidth}
	\centering
	\includegraphics[width=\linewidth]{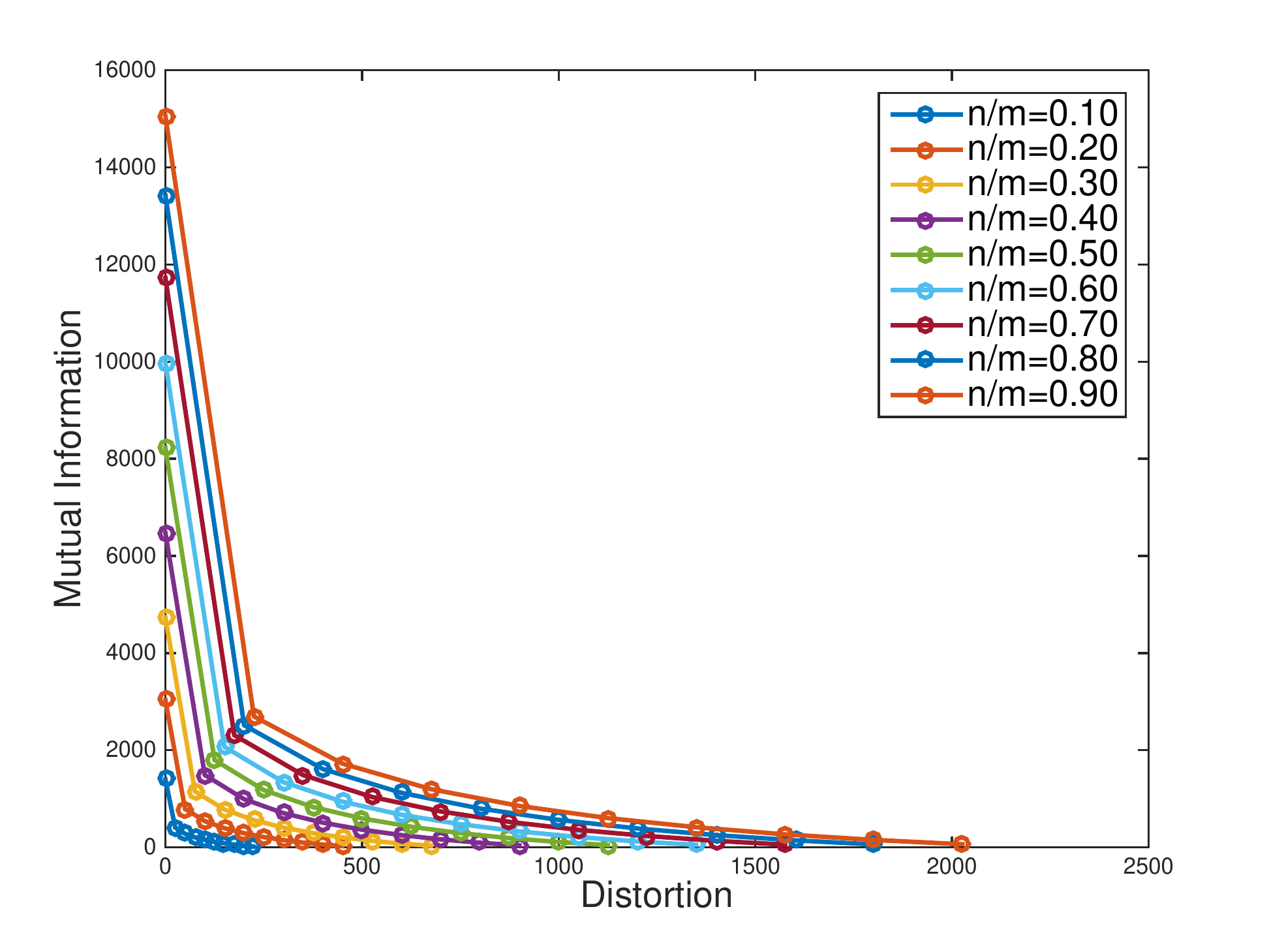}
	\caption{Dimension $m$ of $U$: 2500}
\end{subfigure}%
	\caption{Minimum mutual information under different distortion budgets in Gaussian multiple-copies case: $n<m$.}\label{Fig:MultipleCopiesGaussian}
\end{figure}

\begin{figure}[!htb]
	\centering
	\begin{subfigure}[b]{0.5\textwidth}
		\centering
		\includegraphics[width=\linewidth]{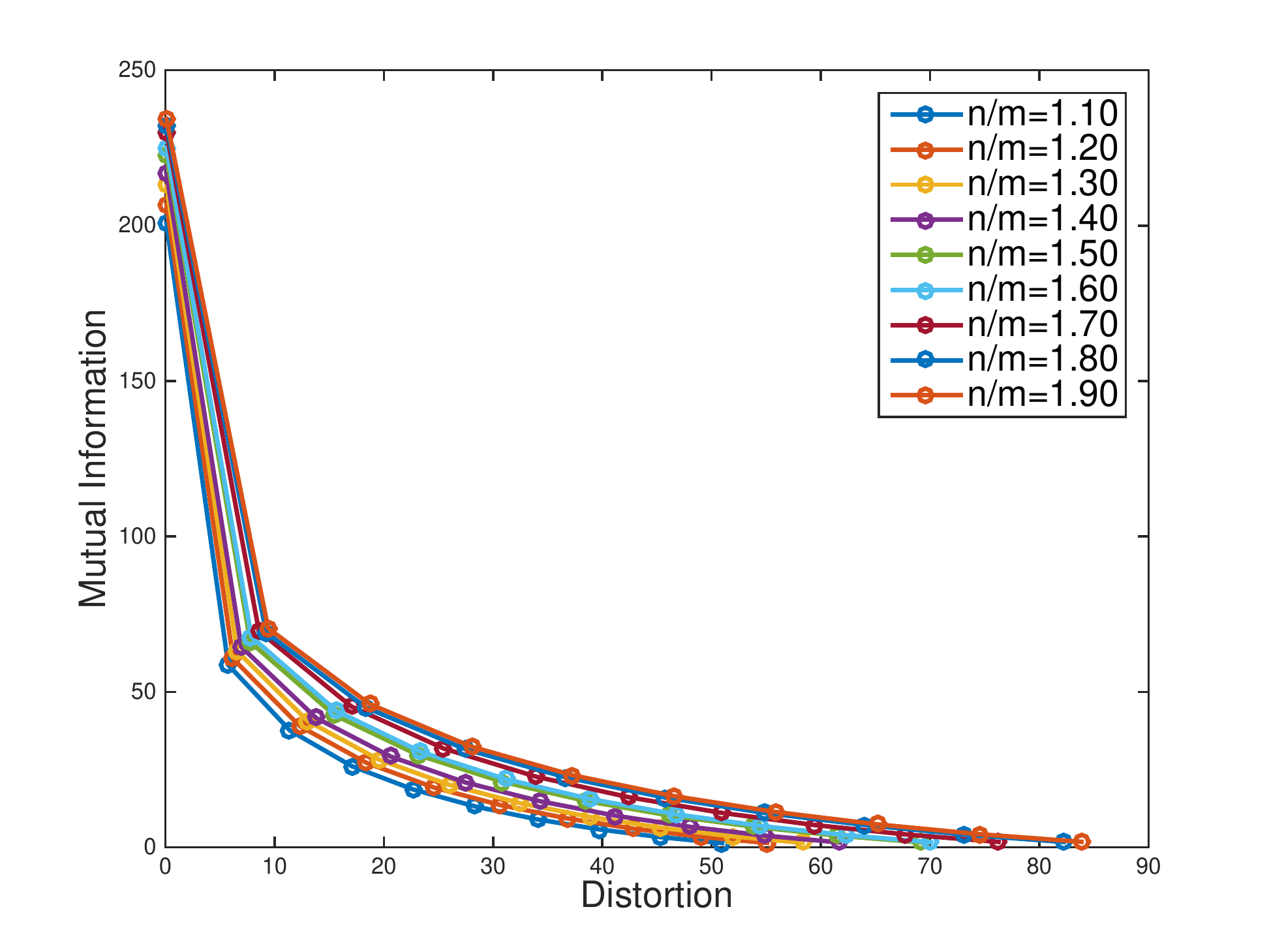}
		\caption{Dimension $m$ of $U$: 50}
	\end{subfigure}%
	~
	\begin{subfigure}[b]{0.5\textwidth}
		\centering
		\includegraphics[width=\linewidth]{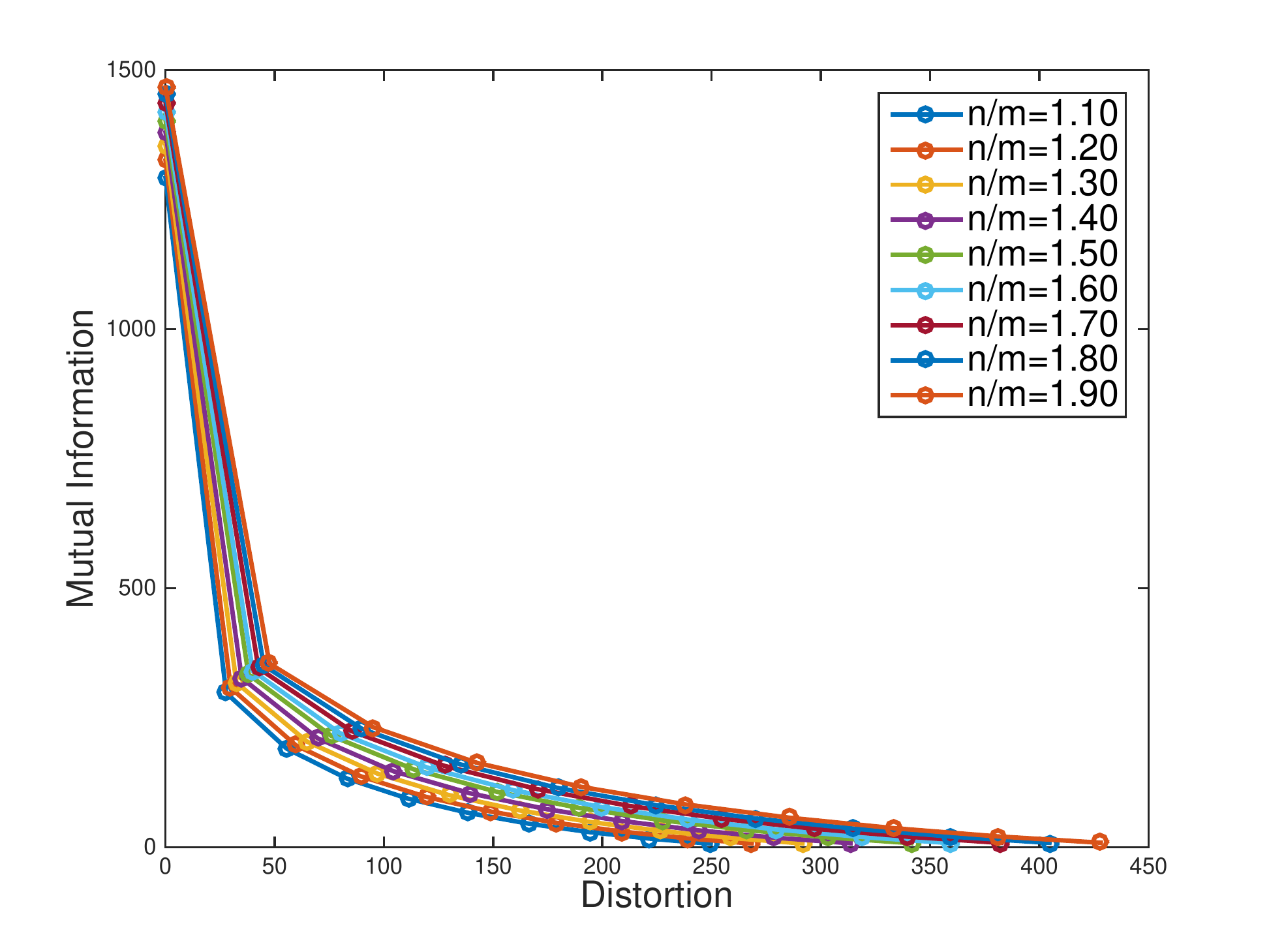}
		\caption{Dimension $m$ of $U$: 250}
	\end{subfigure}
	\begin{subfigure}[b]{0.5\textwidth}
	\centering
	\includegraphics[width=\linewidth]{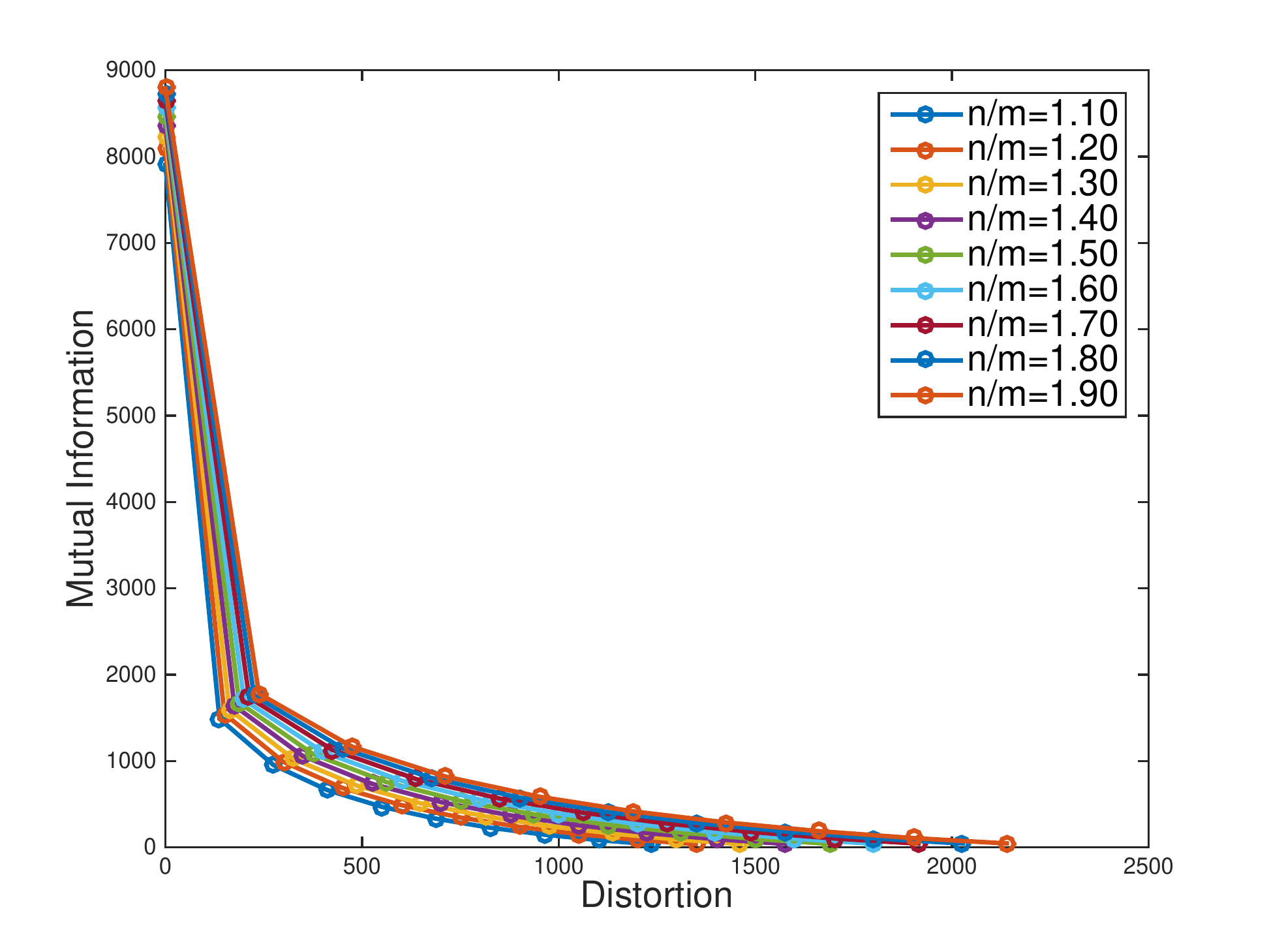}
	\caption{Dimension $m$ of $U$: 1250}
\end{subfigure}%
~
\begin{subfigure}[b]{0.5\textwidth}
	\centering
	\includegraphics[width=\linewidth]{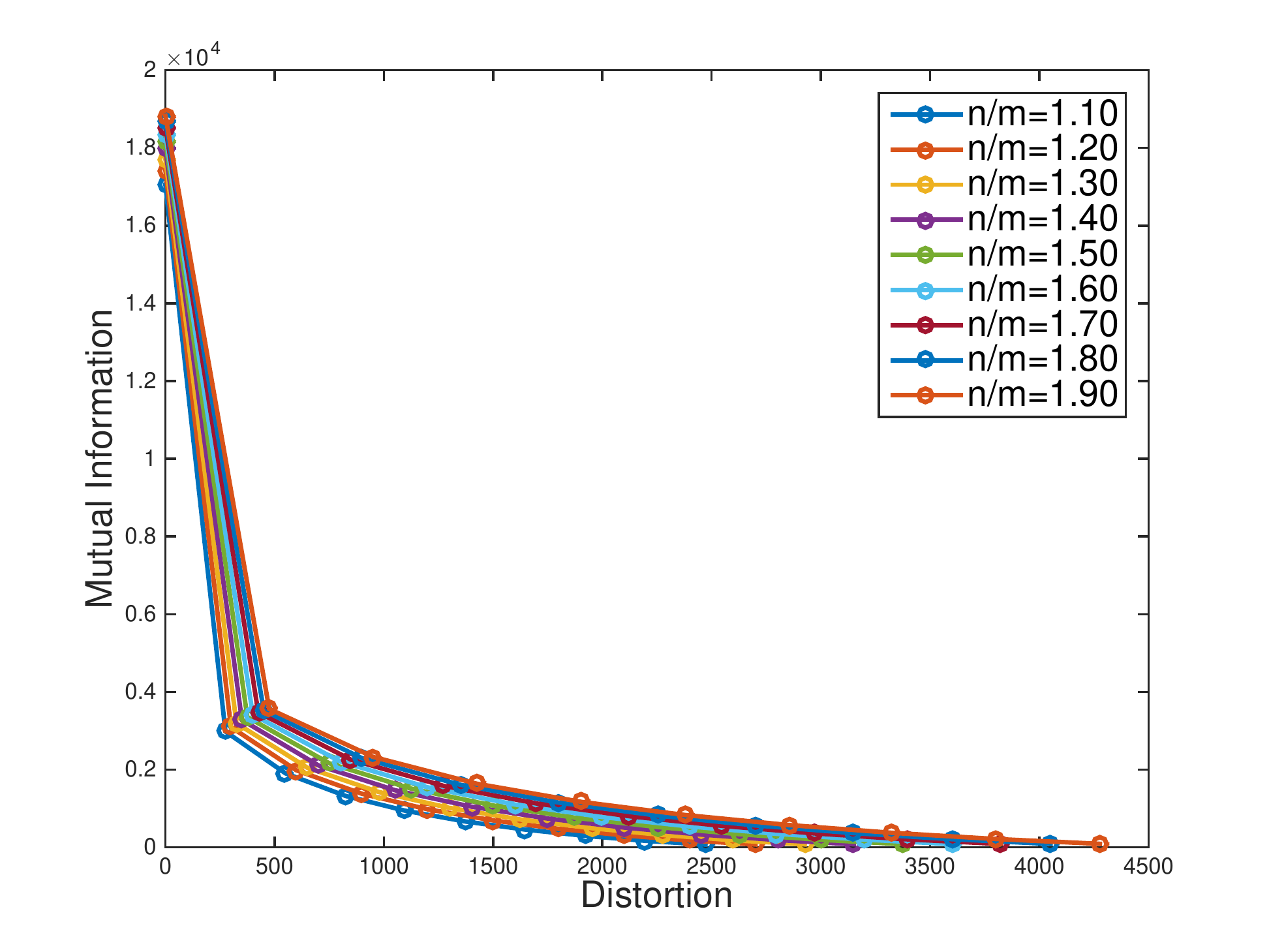}
	\caption{Dimension $m$ of $U$: 2500}
\end{subfigure}%
	\caption{Minimum mutual information under different distortion budgets in Gaussian multiple-copies case: $n\geq m$.}\label{Fig:MultipleCopiesGaussianTall}
\end{figure}

From the results, we can see that as the distortion increases up to the sum of the squared singular values of the projection matrix, the mutual information can go to zero. It also shows that for a given distortion $D$, when the dimensionality of the space to which the source $U\in\mbR^m$ is projected increases, the mutual information will also increase, meaning that an adversarial attack on the mutual information is more difficult to achieve. Besides, this improved robustness is more obvious for fat projection matrix with $n<m$. For example, in the case where $m=50$ and $D=2.5$, and the matrix is fat, increasing $n/m$ from 0.1 to 0.9 gives an increase in mutual information from about 6 to about 120, i.e., 1900\% improvement. While in the case where $m=50$ and $D=2.5$, as $n/m$ increases from 1.1 to 1.9, the mutual information increases from about 150 to 200, i.e., 33\% improvement. However, no matter what type the matrix is, the mutual information increases monotonically with respect to the increase of $n$, which provides empirical evidence for our analysis in Section \ref{Sec:GaussianWithProjection}.




We now present empirical results for minimizing mutual information \eqref{Defn:BSCAdversarialAttackExplicit} in the binary symmetric channel case. For each source distribution $a,b,c,d$, we minimize the mutual information $I(U;Y)$ in \eqref{Eq:BSCMutualInfo} with respect to conditional distribution $p(E|U,X)$ where $Y=X+E$ and the addition is modular over 2. Due to the small scale of the problem, we can adopt a grid search method to find the optimal distribution $p_1,p_2,p_3,p_4$, i.e., constructing a 4D cube $[0,1]^4$ and drawing 100 equally spacing lines along each dimension to obtain a gridded cube which has $10^8$ points. We first perform distortion check to remove those points which may violate the distortion constraint, and then find the optimal solution from the rest points whose indices are the optimal distribution. We consider different source distributions, e.g., $[a,b,c,d]=[0.45,0.05,0.05,0.45]$, $[0.4,0.1,0.1,0.4]$, $[0.3,0.1,0.1,0.3]$, and $[0.25,0.25,0.25,0.25]$. For each source distribution, we find the minimal mutual information for given distortion budget $D$. The results are shown in Figure \ref{Fig:BSCMIvsDistortion}.

\begin{figure}[!htb]
	\centering
	\begin{subfigure}[b]{0.5\textwidth}
		\centering
		\includegraphics[width=\linewidth]{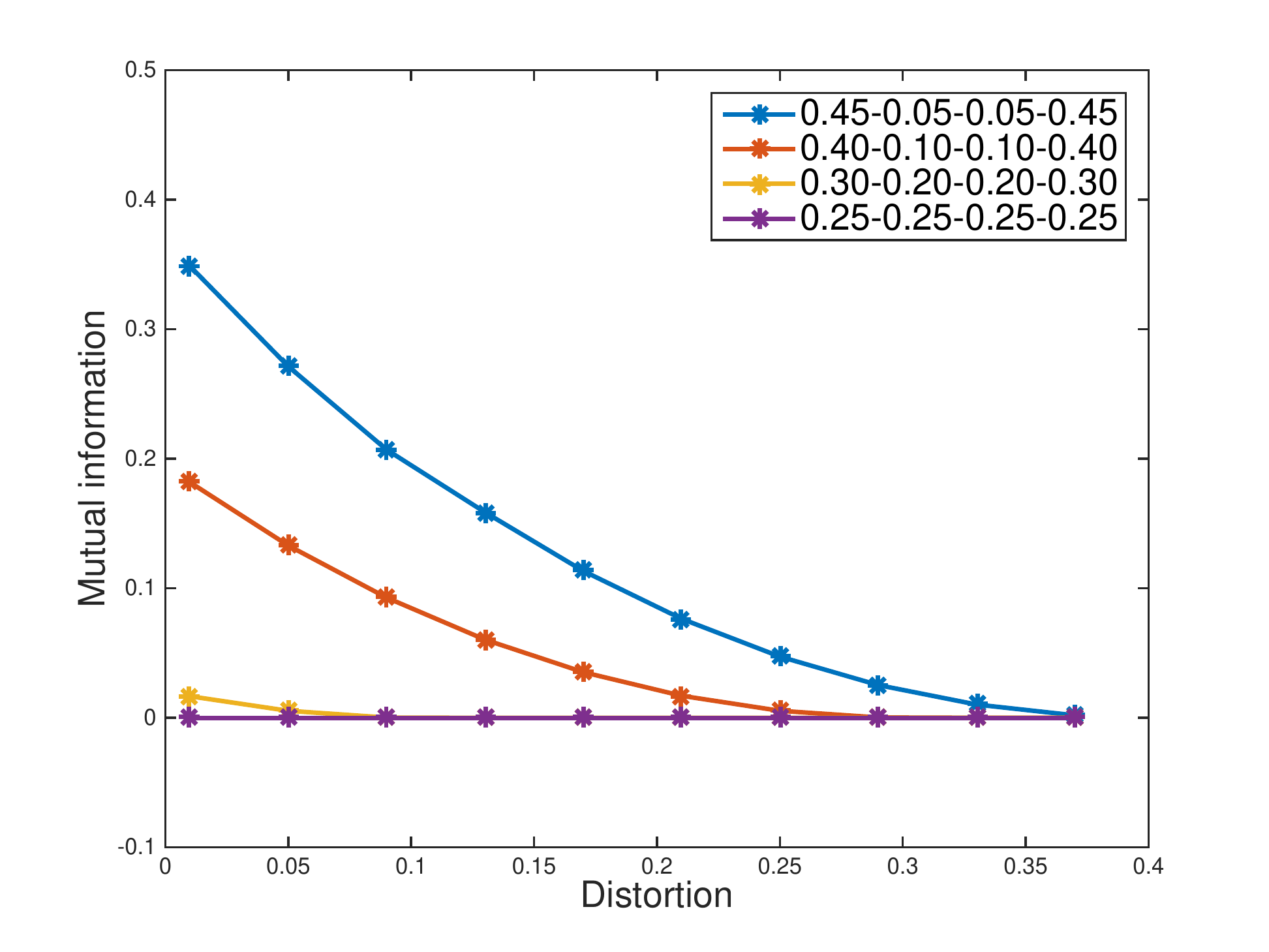}
		\caption{}\label{Fig:Equal}
	\end{subfigure}%
	~
	\begin{subfigure}[b]{0.5\textwidth}
		\centering
		\includegraphics[width=\linewidth]{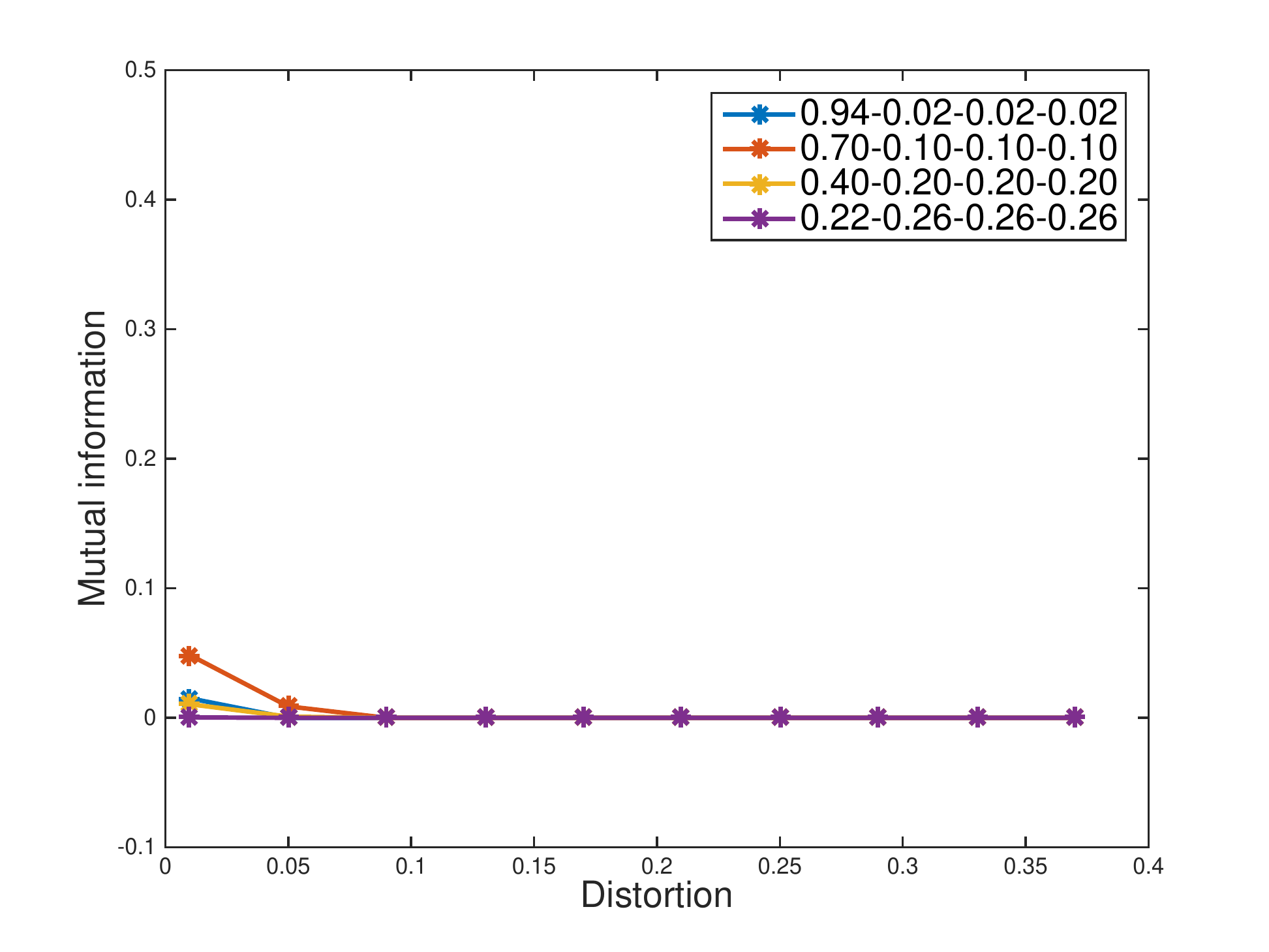}
		\caption{}\label{Fig:Concen1}
	\end{subfigure}
	\begin{subfigure}[b]{0.5\textwidth}
	\centering
	\includegraphics[width=\linewidth]{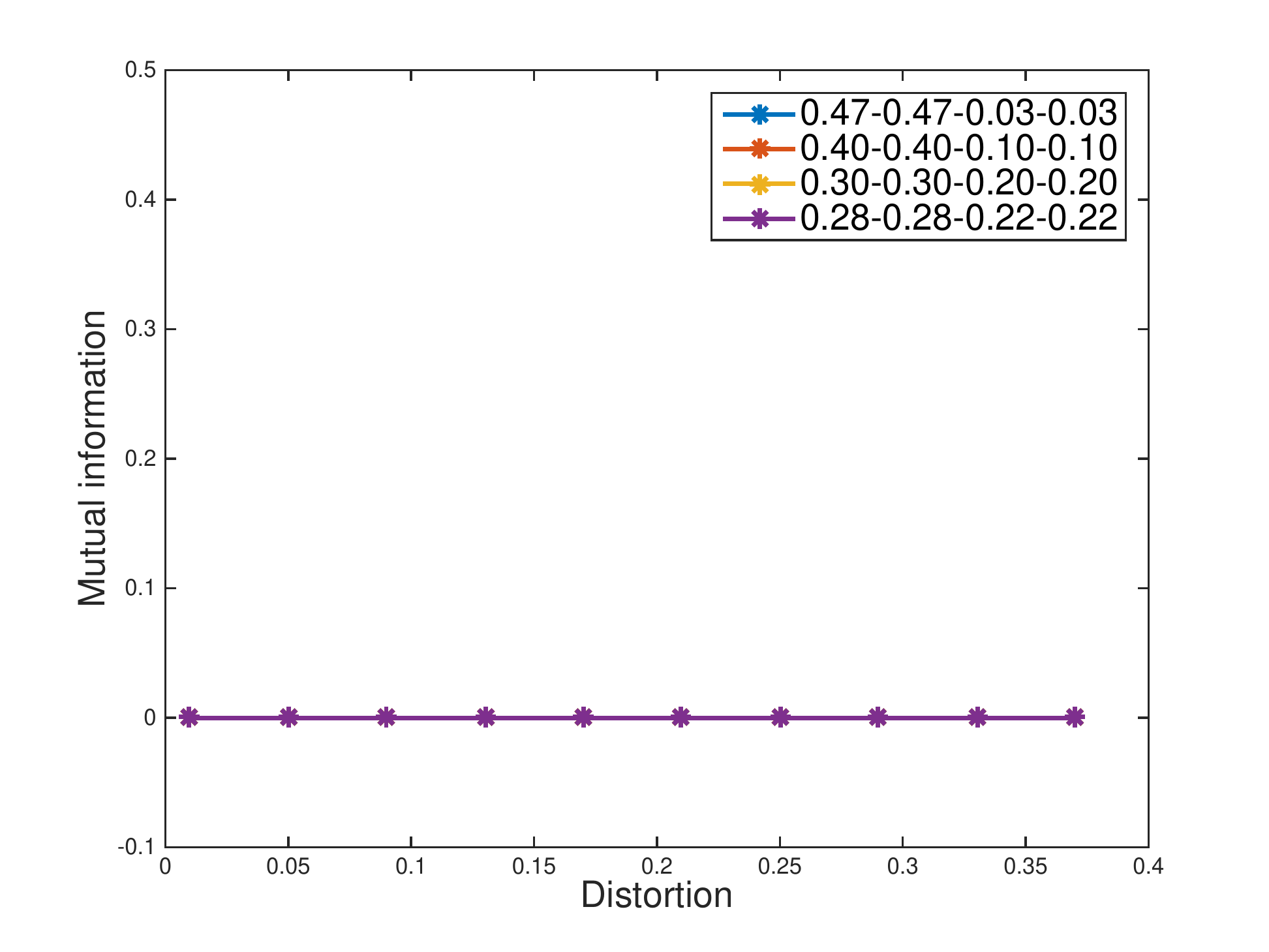}
	\caption{}\label{Fig:Concen2}
\end{subfigure}%
	\caption{Minimum mutual information under different distortion budgets: $a-b-c-d$ is the joint distribution of $(U,X)$.}\label{Fig:BSCMIvsDistortion}
\end{figure}

From the results, we can see that when the source joint distribution of $(U,X)$ is highly concentrated or uniform, we need very small distortion to reduce the mutual information to be 0. For example, in Figure \ref{Fig:Concen1}, when $\mbP(U=0,X=0)=0.94$, we only need about 0.05 distortion to set the mutual information to be 0. In Figure \ref{Fig:Concen2}, when $\mbP(X=0|U=0)+\mbP(X=1|U=0)=0.94$, there is almost no mutual information  between $U$ and $X+E$ even for distortion close to 0. In Figure \ref{Fig:Equal}, when the source joint distribution is uniform, the mutual information is zero even for very small distortion. In the general situation, increasing the distortion can indeed reduce the mutual information, e.g., the case where the source joint distribution is $a=0.45, b=0.05, c=0.05, d=0.45$ in Figure \ref{Fig:Equal}.


\bibliography{02Ref_MI-Based_Attack}
\bibliographystyle{unsrt}
\end{document}